\documentclass[12pt,verbose=true,letterpaper,margin=1.5in]{article}
\usepackage{arxiv}

\usepackage[utf8]{inputenc}
\usepackage[linesnumbered,ruled]{algorithm2e}
\usepackage{amsmath,amsthm,amsfonts,mathtools}
\usepackage{tikz}

\usepackage{booktabs} % testing alternative tables
\allowdisplaybreaks % to be able to write long equations on several pages

\newcommand{\gsemo}{GSEMO\xspace}
\newcommand{\semo}{SEMO\xspace}
\newcommand{\gsemod}{GSEMO$_D$\xspace}
\newcommand{\semod}{SEMO$_D$\xspace}
\newcommand{\mpoga}{\mbox{$(\mu+1)$~GA}\xspace}
\newcommand{\oneminmax}{\textsc{OneMinMax}\xspace}

\newcommand{\jump}{\textsc{Jump}\xspace}
\newcommand{\N}{{\mathbb N}}
\newcommand{\Jhot}{J_\text{hot}}
\newcommand{\Dopt}{D_\text{opt}}
\newcommand{\Ds}{D_{S_3}}
\newcommand{\Deop}{D_\text{EoP}}
\newcommand{\popt}{p_\text{opt}}
\newcommand{\ps}{p_{S_3}}
\newcommand{\peop}{p_\text{EoP}}

\DeclareMathOperator{\Geom}{Geom}
\usepackage{adjustbox}
\newtheorem{theorem}{Theorem}
\newtheorem{lemma}{Lemma}
\newtheorem{corollary}{Corollary}

\begin{document}

\author{Denis Antipov
\\Optimisation and Logistics\\
School of Computer Science
\\The University of Adelaide
\\Adelaide, Australia
\And
Aneta Neumann
\\Optimisation and Logistics\\
School of Computer Science
\\The University of Adelaide
\\Adelaide, Australia
\And
Frank Neumann
\\Optimisation and Logistics\\
School of Computer Science
\\The University of Adelaide
\\Adelaide, Australia
}

\title{Rigorous Runtime Analysis of Diversity Optimization with GSEMO on OneMinMax}

\maketitle

\begin{abstract}
  The evolutionary diversity optimization aims at finding a diverse set of solutions which satisfy some constraint on their fitness. In the context of multi-objective optimization this constraint can require solutions to be Pareto-optimal. In this paper we study how the \gsemo algorithm with additional diversity-enhancing heuristic optimizes a diversity of its population on a bi-objective benchmark problem \oneminmax, for which all solutions are Pareto-optimal.

  We provide a rigorous runtime analysis of the last step of the optimization, when the algorithm starts with a population with a second-best diversity, and prove that it finds a population with optimal diversity in expected time $O(n^2)$, when the problem size $n$ is odd. For reaching our goal, we analyse the random walk of the population, which reflects the frequency of changes in the population and their outcomes.
\end{abstract}


\section{Introduction}
\label{sec:intro}

Diversity optimization is an area of optimization, where we aim at finding a set of solutions which all have a good quality and where the set is diverse. It is widely used in practice, e.g., in the area of quality diversity (QD). QD is a new paradigm, which aims at finding a set of high-quality solutions that differ based on certain user-defined features \cite{DBLP:conf/alife/LehmanS08,DBLP:journals/corr/MouretC15}. QD algorithms have been successfully applied to the area of robotics \cite{DBLP:journals/ras/KimCD21,DBLP:conf/gecco/RakicevicCK21}, design \cite{DBLP:conf/ppsn/HaggAB18}, and games \cite{DBLP:journals/tciaig/AlvarezDFT22,DBLP:conf/aaai/FontaineLKMTHN21}.

Diversity optimization problems in which the goal is to find the most diverse (according to some measure) set of solutions, all of which satisfy some quality constraint, are usually much harder than the same problems in which the goal is to find a single best solution. The main reason for this is the high dimensionality of the search space: instead of the original search space we have to perform the optimization in the space of population.
Often such problems are solved with evolutionary algorithms (EAs), which are considered as a good general-purpose solvers for such high-dimensional search spaces. This approach is called the evolutionary diversity optimization (EDO) and has been shown to be effective
to evolve high quality diverse sets of solutions for the traveling salesperson problem
\cite{DBLP:journals/ec/GaoNN21,DBLP:conf/gecco/DoBN020,DBLP:conf/foga/NikfarjamB0N21,DBLP:conf/gecco/NikfarjamBN021}, the knapsack problem \cite{DBLP:conf/gecco/BossekN021}, and minimum spanning tree problems \cite{DBLP:conf/gecco/Bossek021}. EDO approach for constrained monotone submodular functions has been introduced in \cite{DBLP:conf/gecco/NeumannB021} to improve the initial population diversity obtained by a diversifying greedy sampling technique. A co-evolutionary approach for evolving an optimized population and interacting with a  diversifying population has been introduced in~\cite{DBLP:conf/gecco/NeumannA022}.
Furthermore, EDO algorithms were also recently utilized for constructing wireless communication networks that minimize the area covered by the senders' transmissions while avoiding adversaries \cite{DBLP:conf/gecco/Neumann23}.

There are different ways to measure diversity. %\todo{citations with examples of diversity measures}
The early study of Ulrich and Thiele~\cite{DBLP:conf/gecco/UlrichT11} used the Solow-Polasky measure to measure the diversity of a population. Later studies that evolved solutions according to a set of features used the discrepancy measure~\cite{DBLP:conf/gecco/NeumannGDN018} as well as population indicators such as the hypervolume indicator and inverted generational distance indicator from the area of evolutionary multi-objective optimization~\cite{DBLP:conf/gecco/NeumannG0019}.

In recent years, the analysis of evolutionary algorithms for computing diverse sets of solutions has become one of the hot topics in the theory of evolutionary computation. In particular, results have been obtained for computing diverse solutions for the traveling salesperson problem and the quadratic assignment problem for the basic case where there is no quality constraint is imposed on the desired solutions~\cite{DBLP:journals/telo/DoGNN22}. The EDO approach for constrained monotone submodular functions~\cite{DBLP:conf/gecco/NeumannB021} builds on sampling greedy approaches that provably come with good approximation guarantees which directly translate to the EDO algorithms. For the classical knapsack problem, QD approaches have been presented in~\cite{DBLP:conf/ppsn/NikfarjamDN22} that resemble dynamic programming and also provide a fully randomized polynomial time approximation scheme (FPRAS).

In the context of the multi-objective optimization the diversity optimization is used to get a diverse set of Pareto-optimal solutions. Doerr et al. in~\cite{DBLP:conf/gecco/DoerrGN16} studied how the $(\mu + 1)$-SIBEA with population size $\mu = n + 1$ finds a diverse population covering the whole Pareto front of \oneminmax benchmark problem, when it uses a heuristics to support the diversity and showed a $O(n^3\log(n))$ upper bound on the runtime. The heuristic used in that paper helps the algorithm decide which of the two individuals with the identical fitness to kick out from the population at all times when this conflict arises, and it always removes the individual with the smaller contribution to the diversity of the population. This heuristic is very similar to the one used in single-objective optimization in~\cite{DangFKKLOSS16,DangFKKLOSS18}, where the authors used the same tie-breaking rule in the $(\mu + 1)$~GA, which resulted in a significant improvement of the algorithm's ability to escape local optima of \jump benchmark functions.

In this paper, we contribute to better understanding of the population's behavior when optimizing the diversity on multiobjective problems. We study \oneminmax, that is, the same problem as in~\cite{DBLP:conf/gecco/DoerrGN16}, but optimized by a slightly different algorithm, the \gsemo. The behavior of these two algorithms after they cover the whole Pareto front of \oneminmax is identical (see Section~\ref{sec:algo} for details), but the perspective of \gsemo simplifies the description of the process. We focus on the last step of the optimization, that is, we assume that the \gsemo starts with a population with the second-best diversity. We prove a $O(n^2)$ bound on the runtime, which improves the results of~\cite{DBLP:conf/gecco/DoerrGN16}, which implies the $O(n^3)$ bound on this stage of the optimization. The previous $O(n^3)$ bound comes from an observation that there is always a two-bits flip present in the population which allows us to find the optimally diverse population and which can be made with probability $\Theta(\frac{1}{n^3})$. We show that during this last stage, the population of the \gsemo performs a random walk (in the space of populations) and it often gets to the state when we have $\Omega(n)$ of such good two-bits flips, and therefore we have a $O(\frac{1}{n^2})$ probability to make one of them.

As a diversity measure we consider the total Hamming distance, that is, the sum of Hamming distances between all pairs of individuals in the population. We note that our result also  holds for any measure which can be computed based on the known number of ones and zeros in each position in the population.

The rest of the paper is organized as follows. In Section~\ref{sec:prelims}, we describe formally the problem we are studying and also present some preliminary results. In Section~\ref{sec:replace-prob}, we prove the bounds on the probability to change an individual in the population of \gsemo. Then, in Section~\ref{sec:runtime}, we analyse the random walk of the population and prove our main result. 
Finally, in Section~\ref{sec:conclusion}, we discuss our results and the further direction of the research.

\section{Preliminaries}
\label{sec:prelims}

In this section, we describe the \gsemod algorithm, the \oneminmax problem and the diversity measure that we aim to optimize. We also define the main problem studied in the paper and state some preliminary results that help explain our goals. Additionally, we provide several auxiliary results that help to reach those goals.

\subsection{The \gsemod}
\label{sec:algo}

The simple evolutionary multiobjective optimizer (\semo) is a multi-objective optimization algorithm which aims at finding a Pareto-optimal population. It is based on the \emph{dominance} relation defined on the elements of the search space, which we define as follows. For elements $x$, $y$ of the search space and for a $k$-objective fitness function $f = (f_1, \dots, f_k)$ defined on that space we say that $x$ dominates $y$ (and write $x \succeq y$), iff for all $i \in [1..k]$ (where $[1..k]$ stands for an integer interval from $1$ to $k$) we have $f_i(x) \ge f_i(y)$ and there exists $i \in [1..k]$ for which $f_i(x) > f_i(y)$.

The \semo starts with a population consisting of one individual which is chosen uniformly at random from the search space. In each iteration it creates a new individual $y$ by choosing a parent uniformly at random from its population and applying the mutation operator to it. If there is no individual in the population which dominates $y$, then we add $y$ into the population and remove all the individuals which are dominated by $y$. If there is an individual $x$ with exactly the same fitness as $y$, then the standard \semo algorithm removes $x$, giving the priority to the newer individual $y$, which supports the exploration of the search space. However, if we have any additional objective to optimize, we can use different rules to decide whether we should remove $x$ or $y$. In this paper we aim at optimizing diversity of the population (we define the measure of diversity later in Subsection~\ref{sec:diversity-measure}), hence if after adding $y$ to the population we have two individuals with the same fitness, we remove the one with the lowest contribution to the diversity. If this contribution is the same, we remove $x$ to enhance the exploration ability of the algorithm. The \semo, which uses this additional mechanism with some diversity measure $D$ is denoted as the \semod.
We note that similar ideas of enhancing diversity have been used in the single-objective optimization, e.g., in the \mpoga in~\cite{DangFKKLOSS16,DangFKKLOSS18}, where a tie-breaking rule which prioritizes some diversity measure of the population helped to escape local optima.

In this paper we focus on the bi-objective pseudo-Boolean optimization, that is, our search space is the space of bit strings of a fixed length $n$ (which is called the \emph{problem size}). As the mutation operator used to create new individuals we consider the standard bit mutation, which flips each bit independently from other bits with probability $\frac{1}{n}$. We follow the common notation and call the \semod with the standard bit mutation the \emph{Global} \semod (\gsemod for brevity).  We also denote the population of the \gsemod in the beginning of iteration $t$ by $P_t$.

The pseudocode of the \gsemod is shown in Algorithm~\ref{alg:gsemod}.

\begin{algorithm}[tp]
    Choose $x \in \{0,1\}^n$ uniformly at random\;
   % determine $g(x)$\;
    $P\leftarrow \{x\}$\;
   
   \Repeat{$stop$}{
   Choose $x\in P$ uniformly at random\;
   Create $y$ by flipping each bit $x_{i}$ of $x$ with probability $\frac{1}{n}$\;
   determine $g(y)$\;
   \If{$\exists w \in P: g(w)=g(y)$}{
   \If{$D(P) \leq D((P \cup \{y\}) \setminus \{w\})$}{$P \leftarrow (P  \cup \{y\}) \setminus \{w\}$}}
   %determine $g(y)$\; include $y$ into $P$\;
   \ElseIf{$\not\exists w \in P: w \succeq y$} {
     $P \leftarrow (P \cup \{y\})\backslash \{z\in P \mid y \succeq z\};$}
       }
   \caption{Global SEMO$_D$ maximizing a multi-objective function $g$ and diversity measured by $D$.} \label{alg:gsemod}
   \end{algorithm}

Although we consider a different algorithm than in~\cite{DBLP:conf/gecco/DoerrGN16}, we note that once the \gsemod and the $(\mu + 1)$-SIBEA$_D$ cover the whole Pareto front of \oneminmax, they both always have a population of $n + 1$ individuals with different fitness values, and they improve the diversity via the tie-breaking rule. Hence, given the same initial population which covers the whole Pareto front, both algorithms are described by the same stochastic process. This first population which covers the Pareto front can have a different distribution for the two algorithms with a random initialization, but this is not important in this paper, since we only consider the last stage of the optimization. In this light we find it easier to use notation of \gsemo to ease the reading, while our results can still be compared to the ones from ~\cite{DBLP:conf/gecco/DoerrGN16}.

\subsection{\oneminmax Problem}
\label{sec:oneminmax}

The \oneminmax problem is a benchmark bi-objective problem, which is defined on bit strings of length $n$ as 
\begin{align*}
    \oneminmax(x) = (|x|, n - |x|),
\end{align*}
where $|x|$ stands for the number of one-bits in $x$. In other words, the first objective is the number of one-bits in $x$ and the second objective is the number of zero-bits in $x$. With this fitness, none of bit strings dominates any other bit string, since if there are $x$ and $y$ such that $|x| > |y|$, then we have $(n - |x|) < (n - |y|)$. Therefore, when we aim at maximizing both objectives, then the whole search space lies on the Pareto front. The size of Pareto front (that is, the number of different fitness values) is $n + 1$. These observations make this problem a good benchmark for studying the diversity optimization.

When we have a population $P$ in which all individuals have different \oneminmax value, by $x_i$ we denote an individual of this population with $\oneminmax(x_i) = (i, n - i)$, that is, $i$ is the number of one-bits in $x_i$.

\subsection{Diversity Measure: Total Hamming Distance}
\label{sec:diversity-measure}

In this paper we consider the \emph{total Hamming distance}, which is a diversity measure for the populations consisting of bit strings, which is equal to the sum of Hamming distances between each pair of individuals in the population. More formally, for population $P = \{x_1, \dots, x_m\}$ consisting of $m$ bit strings the total Hamming distance is
\begin{align*}
    D(P) = \sum_{i = 1}^{m - 1} \sum_{j = i + 1}^m H(x_i, x_j),
\end{align*}
where $H(\cdot,\cdot)$ is the Hamming distance between two bit strings. In our analysis we use the following observation, which allows us to compute the diversity based only on the number of one-bits in each position.
\begin{lemma}
    \label{lem:hamming-alternative}
    Let $P$ be a population of size $m$. For all $k \in [1..n]$ let $m_k$ be the number of individuals of population $P$ which have a one-bit in position $k$. Then the total Hamming distance of the population is
    \begin{align*}
        \sum_{k = 1}^n m_k (m - m_k),
    \end{align*}
    and it is maximized when for all $m_k$, $k = 1, \dots,n$, are equal to $\frac{m \pm 1}{2}$, when $m$ is odd, and are equal to $\frac{m}{2}$, when $m$ is even. 
\end{lemma}
\begin{proof}
    For any two individuals $x$ and $y$ and for all $k \in [1..n]$ let $h_k(x, y)$ be one, if $x$ and $y$ have different values in position $k$ and zero otherwise. Then we have $H(x, y) = \sum_{k = 1}^n h_k(x, y)$. The total Hamming distance is therefore
    \begin{align*}
        D(P) = \sum_{i = 1}^{m - 1} \sum_{j = i + 1}^m \sum_{k = 1}^n h_k(x_i, x_j) = \sum_{k = 1}^n \sum_{i = 1}^{m - 1} \sum_{j = i + 1}^m h_k(x_i, x_j).
    \end{align*} 
    Consider the two inner sums (over $i$ and $j$) for some particular position $k$. A pair of individuals $(x_i, x_j)$ contributes one to this sum if and only if these two individuals have different bits in position $k$, and there are exactly $m_k (m - m_k)$ such pairs. Hence, we have
    \begin{align*}
        D(P) = \sum_{k = 1}^n m_k (m - m_k).
    \end{align*} 
    Each term of this sum is a quadratic function of $m_k$, which is maximized when $m_k = \frac{m}{2}$. If $m$ is even, then $\frac{m}{2}$ is also an integer, and therefore we can have all $m_k = \frac{m}{2}$, which maximizes the diversity. Otherwise, each term reaches its largest value when $m_k = \frac{m \pm 1}{2}$.
\end{proof}

From this lemma it trivially follows that for even $m$ the maximal diversity is $\frac{m^2n}{4}$ and for odd $m$ it is $\frac{(m - 1)(m + 1)n}{4} = \frac{n(m^2 - 1)}{4}$. %We also note that the maximal diversity monotonically increases with the growth of population size.

\subsection{Problem Statement}
\label{sec:problem-statement}

In this paper we study, how the \gsemod optimizes \oneminmax when it uses the total Hamming distance as a diversity measure to break the ties between the individuals with the same fitness. By the \emph{runtime} we denote the number of iterations made by \gsemod before it finds a population, which (i) covers the whole Pareto front and (ii) has the maximal diversity.
This problem has already been studied in~\cite{DBLP:conf/gecco/DoerrGN16}, where it was shown that the expected runtime is $O(n^3\log(n))$. 

In this paper we study the algorithm's behavior in the very last stage of the optimization, that is, when we start with the population covering the whole Pareto front which has the second-best diversity value. The results of~\cite{DBLP:conf/gecco/DoerrGN16} imply that the expected runtime with this initialization is $O(n^3)$. With a rigorous analysis of the population's dynamics we improve this upper bound and show that the expected runtime is $O(n^2)$.

Lemma~\ref{lem:hamming-alternative} suggest that the set of populations is very different for the even and odd values of $n$. When $n$ is odd, the population size $n + 1$ is even, and to reach the optimal diversity we must have exactly $\frac{n + 1}{2}$ one-bits in each position. For even $n$ (and thus, odd $n + 1$) we have two options for each position, since we can have $\frac{n + 1}{2} \pm \frac{1}{2}$ one-bits in each position. The latter case gives us more freedom, hence intuitively it should be easier for the \gsemod. For this reason in this paper we consider only the case when $n$ is odd, which is harder for the algorithm (however, we believe that it is easier for the analysis).

To show that for all odd $n$ there exists at least one population with the optimal diversity, we build such population as follows. For each individual $x_i$ with $i < \frac{n - 1}{2}$ we can take any bit strings with exactly $i$ one-bits in it. For larger $i$ we take $x_i$, which is a bit-wise inverse of $x_{n - i}$. This population covers the whole Pareto front and due to the inverse operation it has equal number of one-bits and zero-bits in each position, which yields the optimal diversity.

To show how the population with the second-best diversity looks like, we introduce the following notation for the populations which cover the whole Pareto front of \oneminmax. We call position $k$ \emph{balanced}, if $m_k = \frac{n + 1}{2}$, that is, it has the same number of one-bits and zero-bits. We call it \emph{almost balanced}, if $m_k = \frac{n + 1}{2} \pm 1$, that is we have a minimal deviation from the balanced number of one-bits. In all other cases we call position \emph{unbalanced}. If all positions are balanced, we have the best diversity, therefore, a population with the second-best diversity must have at least one almost balanced or unbalanced position.

The total number of one-bits in any population covering the whole Pareto front of \oneminmax is the same as the number of zero-bits, which follows from the symmetry of \oneminmax. Therefore, when we have a position with more than $\frac{n + 1}{2}$ one-bits, we also have a position with more than $\frac{n + 1}{2}$ zeros. This brings us to conclusion that a population with exactly two almost balanced positions has the second-best diversity: we cannot have only one unbalanced or almost balanced position, and adding new almost balanced positions or making the almost balanced positions unbalanced reduces the total hamming distance.

When the \gsemod has a population covering the whole Pareto front and it generates an offspring with $i$ one-bits, it decides if this offspring should replace $x_i$ or not. Since it does not accept an individual to the population which reduces the diversity, during the whole run until we find a population with the optimal diversity we have exactly two almost balanced positions. For each iteration $t$ we call the only position with $\frac{n + 3}{2}$ one-bits in population $P_t$ \emph{hot} and the only position with $\frac{n - 1}{2}$ one-bits \emph{cold}. Note that this definition depends on the population $P_t$, which can change during the algorithm's run, hence the hot and the cold position can move. Similar to the notation of position, we call a population \emph{balanced}, if all positions are balanced, we call it \emph{almost balanced}, if there are exactly two almost balanced positions and other positions are balanced, and we call a population \emph{unbalanced} otherwise.

\subsection{Useful tools}
\label{sec:tools}

In our proof we use the following auxiliary lemma to estimate the upper bound on the probability of generating some particular individuals.

\begin{lemma}\label{lem:hamming}
    Let $x$ and $y$ be bit strings of length $n$. If the Hamming distance between $x$ and $y$ is $H(x, y) \ge d$, then the probability that the standard bit mutation applied to $x$ generates $y$ is at most $\frac{1}{e(n-1)^d}$. If $d$ is constant when $n$ tends to positive infinity, then this probability is at most $\frac{1 + O(1/n)}{en^d}$.
\end{lemma}

\begin{proof}
    Let the Hamming distance between $x$ and $y$ be $d' \ge d$. Then to generate $y$ via the standard bit mutation applied to $x$ we must flip the $d'$ bits which are different and do not flip any of $n - d'$ bits which are the same in $x$ and $y$. The probability to do it is
    \begin{align*}
        \frac{1}{n^{d'}} \left(1 - \frac{1}{n}\right)^{n - d'} &= \frac{1}{(n - 1)^{d'}} \left(1 - \frac{1}{n}\right)^n \le \frac{1}{e(n - 1)^{d'}} \\
        &\le \frac{1}{e(n - 1)^d} = \frac{1}{en^d} \left(\frac{n - 1}{n}\right)^{-d}.
    \end{align*} 

    If $d$ is $\Theta(1)$ when $n \to +\infty$, then by Bernoulli inequality we have 
    \begin{align*}
        \left(1 - \frac{1}{n}\right)^{-d} &= \frac{1}{\left(1 - \frac{1}{n}\right)^d} \le \frac{1}{1 - \frac{d}{n}} = \frac{1 + \frac{2d}{n}}{\left(1 - \frac{d}{n}\right)\left(1 + \frac{2d}{n}\right)} \\
        &= \frac{1 + \frac{2d}{n}}{1 + \frac{d}{n} - \frac{2d^2}{n^2}} \le 1 + \frac{2d}{n} = 1 + O\left(\frac{1}{n}\right),
    \end{align*}
    where the last inequality holds when $\frac{d}{n} \le \frac{1}{2}$, which is true when $n$ is large enough, since $d$ is a constant. Therefore, the probability that the standard bit mutation applied to $x$ results in $y$ is at most $\frac{1}{en^d}(1 + O(\frac{1}{n}))$.
\end{proof}

In our analysis we split the algorithm run into phases, which can result either in a success or in a failure. The following lemma helps us to estimate the expected length of each phase and also helps us to estimate the probability that it ends successfully.

\begin{lemma}\label{lem:phase}
 Consider a sequence of random trials $\{X_t\}_{t \in \N}$ (not necessarily independent), where each trial results in one of three outcomes $\{\omega_1, \omega_2, \omega_3\}$. For all $t \in \N$ let $A_t$ and $B_t$ be the events when $X_t = \omega_1$ and $X_t = \omega_2$ correspondingly. For all $t \in \N \setminus \{1\}$ let $C_t$ be the event that for all $\tau \in [1..t - 1]$ the trial $X_\tau$ resulted in $\omega_3$. Let $p_t = \Pr[A_t \mid C_t]$ for all $t \in \N \setminus \{1\}$ and $p_1 = \Pr[A_1]$. Let $q_t = \Pr[B_t \mid C_t]$ for all $t \in \N \setminus \{1\}$ and $q_1 = \Pr[B_1]$. Let $T$ be the minimum $t$ such that $X_t \ne \omega_3$.
 
 If there exist some $p \in (0, 1)$ and $\alpha > 0$ such for all $t \in \N$ we have $p_t + q_t \ge p$ and also $\frac{q_t}{p_t} < \alpha$, then we have $E[T] \le \frac{1}{p}$ and $\Pr[X_T = \omega_1] \ge \frac{1}{1 + \alpha}$. 
\end{lemma}
\begin{proof}
    We have $\Pr[T \ge 1] = 1 - (p_1 + q_1) \le 1 - p$ and for all $t \ge 2$ we have $\Pr[T \ge t] = \Pr[C_{t + 1}] =  \Pr[C_t](1 - \Pr[A_t \cup B_t \mid C_t]) \le \Pr[C_t](1 - p)$. By induction we have that $\Pr[T \ge t] \le (1 - p)^t$, hence $T$ is dominated by the geometric distribution $\Geom(p)$, and thus $E[T] \le \frac{1}{p}$.

    To estimate the probability that $X_T = \omega_1$ we consider some arbitrary $t \in \N$ and condition on $T = t$, which is the same event as $(A_t \cup B_t) \cap C_t$. If $t = 1$, then we have
    \begin{align*}
        \Pr[X_1 = \omega_1 \mid T = 1] = \frac{\Pr[A_1]}{\Pr[A_1 \cup B_1]} = \frac{p_1}{p_1 + q_1} = \frac{1}{1 + \frac{q_1}{p_1}} \ge \frac{1}{1 + \alpha}
    \end{align*}
    
    For all $t \ge 2$, since $A_t \cap C_t$ and $B_t \cap C_t$ are disjoint events, we have
    \begin{align*}
        \Pr[X_T &= \omega_1 \mid T = t] = \frac{\Pr[X_T = \omega_1 \cap T = t]}{\Pr[T = t]} \\
        &= \frac{\Pr[A_t \cap ((A_t \cup B_t) \cap C_t)]}{\Pr[((A_t \cup B_t) \cap C_t)]} 
        = \frac{\Pr[A_t \cap C_t]}{\Pr[(A_t  \cap C_t) \cup (B_t \cap C_t)]} \\
        &= \frac{\Pr[A_t \cap C_t]}{\Pr[A_t  \cap C_t] + \Pr[B_t \cap C_t]} \\
        &= \frac{1}{1 + \frac{\Pr[B_t \cap C_t]}{\Pr[C_t]} \cdot \frac{\Pr[C_t]}{\Pr[A_t \cap C_t]}} 
        = \frac{1}{1 + \frac{q_t}{p_t}} \ge \frac{1}{1 + \alpha}.
    \end{align*} 

    This bound is independent from $t$, hence $\Pr[X_T = \omega_1] \ge \frac{1}{1 + \alpha}$.
\end{proof}

\section{The Probability to Make Changes}
\label{sec:replace-prob}

In this section we consider some particular individual $x_i$ which has $i$ one-bits and estimate the probability that we replace it in one iteration with a different bit string. This probability depends on the values of bits in $x_i$ in the two almost balanced positions. Note that $x_i$ can be replaced only by another individual $x_i'$, which also has $i$ one-bits. We cannot accept an individual which reduces the diversity, hence we can only accept $x_i'$, which makes the two almost balanced positions balanced and keeps other positions balanced as well (then we get the optimal diversity), or we can accept $x_i'$, which makes at least one of the almost balanced positions balanced and makes at most two balanced positions almost balanced. In the latter case we move either the hot, the cold, or both these positions to another place.

The cold position can be balanced or moved to another place only when $x_i$ has a zero-bit in the cold position, and $x_i'$ has a one-bit there. To move it to another position $j$, we must decrease the number of one-bits in $j$, hence we can move it only to a balanced position in which $x_i$ has a one-bit, and $x_i'$ must have a zero-bit there. Similarly, the hot position can be balanced or moved only when $x_i$ has a one-bit in it and $x_i'$ has a zero bit there. It can be moved only to a position, where $x_i$ has a zero-bit, and $x_i'$ must have a one-bit there. These observations trivially imply the following lemma.

\begin{lemma}\label{lem:h0c1}
    Consider an individual $x_i$ with a one-bit in the cold position and a zero-bit in the hot position. If we replace it with any other individual with exactly $i$ one-bits, it will reduce the diversity.
\end{lemma}

\begin{proof}
    By the discussion before the lemma, we cannot balance neither the hot, nor the cold positions by replacing $x_i$, hence any replacement of $x_i$ with a different bit string will either make one of these positions unbalanced or it will add new almost balanced positions. In both cases the diversity is reduced.
\end{proof}

In the following lemma we show, how and with what probability we can replace $x_i$ which has one-bits both in the cold and in the hot positions.

\begin{lemma}\label{lem:h1c1}
    Consider individual $x_i$ with one-bits in both the hot and the cold positions (thus, $i \ge 2$). $x_i$ can be replaced by $n - i$ different bit strings without decreasing the diversity, and all these replacements can only move the hot position to another place, but they can neither move the cold position nor find the optimal population.

    If $x_{i - 1}$ is different from $x_i$ only in the hot position (and has a zero-bit in it), then the probability that we create an individual which can replace $x_i$ without reducing the diversity is in \[
        \left[\frac{n - i}{en^2}\left(1 - O\left(\frac{1}{n}\right)\right), \frac{n - i + 2}{en^2}\left(1 + O\left(\frac{1}{n}\right)\right)\right].
    \] 
    Otherwise this probability is in 
    \[
        \left[\frac{n - i}{en^3}\left(1 - O\left(\frac{1}{n}\right)\right), \frac{7}{en^2}\left(1 + O\left(\frac{1}{n}\right)\right)\right].
    \]
\end{lemma}

\begin{proof}[Proof of Lemma~\ref{lem:h1c1}]
    For $i = n$ the only string which has $i$ one-bits is the all-ones bit string, which cannot be replaced by any other bit string. By the discussion in the beginning of this section, we cannot move the cold position by replacing $x_i$, but we can move the hot position to one of the $n - i$ positions, where $x_i$ has a zero. Hence, the bit string which can replace $x_i$ (and move the hot position) must be different from $x_i$ in exactly two positions: in the hot position and in one of zero-bits of $x_i$. Therefore, there are $n - i$ bit strings different from $x_i$ which can replace it without reducing the diversity. In the rest of the proof we call these bit strings \emph{valid}.

    We now estimate the probability to generate a valid bit string conditional on the number $j$ of one-bits in the parent $x_j$ that we choose. If we choose $x_j$ as a parent with $|i - j| \ge 3$, then the Hamming distance from $x_j$ to any valid bit string is at least $3$. Hence, by Lemma~\ref{lem:hamming} the probability to generate a valid bit string from $x_j$ is at most $\frac{1}{en^3}(1 + O(\frac{1}{n}))$. By the union bound over all $n - i$ valid bit strings, the probability that we generate any of them is at most
    \begin{align*}
        \frac{n - i}{en^3}\left(1 + O\left(\frac{1}{n}\right) \right).
    \end{align*}

    If we choose $x_j$ as a parent with $|i - j| = 2$, then the distance from it to any valid bit string is at least $2$ and by Lemma~\ref{lem:hamming} and by the union bound over $n - i$ valid bit strings, the probability to create a valid bit strings is at most 
    \begin{align*}
        \frac{n - i}{en^2} \left(1 + O\left(\frac{1}{n}\right) \right).
    \end{align*}

    If we choose $x_{i + 1}$ as a parent, then we argue that it can have a distance one to at most two valid bit strings, and the distance to the rest of them is at least $3$. Let $\tilde S_0$ and $\tilde S_1$ be the sets of positions, in which there is a zero-bit and a one-bit in $x_i$ correspondingly. Then each valid bit string has exactly one one-bit in $\tilde S_0$, each in a unique position. It also has $i - 1$ one-bits in $\tilde S_1$ (in all positions, except the hot one). If $x_{i + 1}$ is in distance one from a valid bit string $x_i'$, then it must have one-bits in all position which are ones in $x_i'$. If it in distance one from three or more valid bit strings, then it must have at least $i - 1 + 3 \ge i + 2$ one-bits, but it has only $i + 1$. Hence, there are at most two valid bit strings in distance $1$ from $x_{i + 1}$. The others must have an odd distance to $x_{i + 1}$, hence it is at least $3$. Consequently, by Lemma~\ref{lem:hamming} and by the union bound over all $n - i$ valid bit strings, we generate a valid bit strings with probability at most
    \begin{align*}
        \frac{2}{n} \left(1 + O\left(\frac{1}{n}\right) \right) + \frac{n - i - 2}{n^3} \left(1 + O\left(\frac{1}{n}\right) \right) = \frac{2}{en} \left(1 + O\left(\frac{1}{n}\right) \right).
    \end{align*}

    We now consider the event when we choose $x_{i - 1}$ as a parent and distinguish two cases.
    \textbf{If $x_{i - 1}$ is different from $x_i$ only in the hot position} (and thus has a zero-bit in it), then the Hamming distance from it to any valid bit string is one (they are different in that one-bit of the valid bit string which is in one of the zero-bits positions of $x_i$). By Lemma~\ref{lem:hamming} and by the union bound the probability that we replace $x_i$ is therefore at most
    \begin{align*}
        \frac{n - i}{en}\left(1 + O\left(\frac{1}{n}\right) \right).
    \end{align*}
    To compute the lower bound on generating a valid bit string from $x_{i - 1}$ in this case we note that for this we need to flip any of the $n - i$ zero bits which are not in the hot position and not to flip any other bit, the probability of which is at least 
    \begin{align*}
        \frac{n - i}{n}\left(1 - \frac{1}{n}\right)^{n - 1} \ge \frac{n - i}{en}.
    \end{align*}

    \textbf{If $x_{i - 1}$ is different from $x_i$ ether in a non-hot position or in more than one position.} Assume that it is in distance one to a valid bit string $x_i'$. Then it must have zero-bits in all positions which are zero in $x_i'$ and one additional zero-bit in a position which is one in $x_i'$. Consequently, it is different from any other valid bit string $x_i''$ in at least three bits: this additional zero-bit, the additional one of $x_i'$, which is zero in $x_i$, and in the additional one of $x_i''$. Hence, by Lemma~\ref{lem:hamming} and by the union bound over the valid bit strings, the probability to replace $x_i$ is at most
    \begin{align*}
        \frac{1}{en}\left(1 + O\left(\frac{1}{n}\right) \right) + \frac{n - i - 1}{n^3}\left(1 + O\left(\frac{1}{n}\right) \right) = \frac{1}{en} \left(1 + O\left(\frac{1}{n}\right) \right).
    \end{align*}

    Finally, if we choose $x_i$ as a parent, then the distance from it to any valid bit string is $2$, thus by Lemma~\ref{lem:hamming} the probability to generate any of them is at most
    \begin{align*}
        \frac{n - i}{en^2}\left(1 + O\left(\frac{1}{n}\right) \right)
    \end{align*}
    and this probability is also at least
    \begin{align*}
        \frac{n - i}{n^2}\left(1 - \frac{1}{n}\right)^{n - 2} \ge \frac{n - i}{en^2}.
    \end{align*}

    We now compute the probability to replace $x_i$ with a valid bit string via the law of total probability. Let $j$ be the index of the individual which we choose as a parent. Recall that we choose $j$ u.a.r. from $[0..n]$. Uniting the considered cases we conclude that if $x_i$ is different form $x_{i - 1}$ only in the hot position, then the probability that we replace $x_i$ is at most
    \begin{align*}
        \Pr&[|j - i| \ge 3] \cdot \frac{n - i}{en^3} \left(1 + O\left(\frac{1}{n}\right) \right) \\
        &+ \Pr[|j - i| = 2] \cdot \frac{n - i}{en^2} \left(1 + O\left(\frac{1}{n}\right) \right) \\
        &+ \Pr[j = i - 1] \cdot \frac{n - i}{en}\left(1 + O\left(\frac{1}{n}\right) \right) \\
        &+ \Pr [j = i] \cdot \frac{n - i}{en^2}\left(1 + O\left(\frac{1}{n}\right) \right) \\
        &+ \Pr[j = i + 1] \cdot \frac{2}{en} \left(1 + O\left(\frac{1}{n}\right) \right) \\
        &\le \left(\frac{n - i}{en^3} + \frac{2(n - i)}{en^3} + \frac{n - i}{en^2} + \frac{n - i}{en^3} + \frac{2}{en^2}\right) \left(1 + O\left(\frac{1}{n}\right) \right) \\
        &= \frac{(n - i)\left(1 + \frac{4}{n}\right) + 2}{en^2} \left(1 + O\left(\frac{1}{n}\right) \right) = \frac{n - i + 2}{en^2} \left(1 + O\left(\frac{1}{n}\right) \right)
    \end{align*}
    and at least 
    \begin{align*}
        \Pr[j = i - 1] \cdot \frac{n - i}{en} &= \frac{n - i}{en(n + 1)} = \frac{n - i}{en^2}\left(1 - \frac{1}{n + 1}\right) \\
        &= \frac{n - i}{en^2}\left(1 - O\left(\frac{1}{n}\right)\right).
    \end{align*}

    Otherwise this probability is at most
    \begin{align*}
        \Pr&[|j - i| \ge 3] \cdot \frac{n - i}{en^3} \left(1 + O\left(\frac{1}{n}\right) \right) \\
        &+ \Pr[|j - i| = 2] \cdot \frac{n - i}{en^2} \left(1 + O\left(\frac{1}{n}\right) \right) \\
        &+ \Pr[j = i - 1] \cdot \frac{1}{en}\left(1 + O\left(\frac{1}{n}\right) \right) \\
        &+ \Pr [j = i] \cdot \frac{n - i}{en^2}\left(1 + O\left(\frac{1}{n}\right) \right) \\
        &+ \Pr[j = i + 1] \cdot \frac{2}{en} \left(1 + O\left(\frac{1}{n}\right) \right) \\
        &\le \left(\frac{n - i}{en^3} + \frac{2(n - i)}{en^3} + \frac{1}{en^2} + \frac{n - i}{en^3} + \frac{2}{en^2}\right) \left(1 + O\left(\frac{1}{n}\right) \right) \\
        &\le \frac{7}{en^2} \left(1 + O\left(\frac{1}{n}\right) \right).
    \end{align*}
    and at least 
    \begin{align*}
        \Pr [j = i] \cdot \frac{n - i}{en^2} = \frac{n - i}{en^2(n + 1)} = \frac{n - i}{en^3}\left(1 - O\left(\frac{1}{n}\right)\right).
    \end{align*}
\end{proof}

We also show the similar lemma for the individuals with zero-bits in both cold and hot positions.

\begin{lemma}\label{lem:h0c0}
    Consider individual $x_i$ with zero-bits in both the hot and the cold positions (thus, $i \le n - 2$). If $i = 0$, then we cannot replace $x_i$ with a different bit string. If $i \ge 1$, then by replacing $x_i$ with a different bit string we can only move the cold position to another place, but we can neither move the hot position nor find the optimal population.

    If $x_{i + 1}$ is different from $x_i$ only in the cold position (and has a one-bit in it), then the probability that we create an individual which can replace $x_i$ without reducing the diversity is in 
    \[
        \left[\frac{i}{en^2}\left(1 - O\left(\frac{1}{n}\right)\right), \frac{i + 2}{en^2}\left(1 + O\left(\frac{1}{n}\right)\right)\right].
    \] 
    Otherwise this probability is in 
    \[
        \left[\frac{i}{en^3}\left(1 - O\left(\frac{1}{n}\right)\right), \frac{7}{en^2}\left(1 + O\left(\frac{1}{n}\right)\right)\right].
    \]
\end{lemma}

We omit the proof, since it can be obtained from the proof of Lemma~\ref{lem:h1c1} by swapping the zero-bits with one-bits.
For the last case, when there is a one-bit in the cold position and a zero-bit in the cold position, we split the analysis into two lemmas. The first one shows, what are the bit strings which can replace $x_i$ and the second lemma estimates the probability to generate one of these strings.

\begin{lemma}\label{lem:h1c0-options}
    Consider an individual $x_i$ such that it has a zero-bit in the cold position and a one-bit in the hot position. Let $\tilde x_i$ be the bit string which is different from $x_i$ only in the hot and in the cold positions. Let $S_0$ and $S_1$ be the sets of positions of zero-bits and one-bits in $\tilde x_i$ correspondingly.
     
    Then if $x_i$ is replaced by any string which has exactly one zero-bit in $S_1$ and exactly one one-bit in $S_0$ then the diversity stays the same. If it is replaced by $\tilde x_i$, then the diversity is improved (and thus becomes optimal). If it is replaced by any other bit string with $i$ one-bits, the diversity is reduced.
\end{lemma}

\begin{proof}
    If we replace $x_i$ with $\tilde x_i$, then we reduce the number of ones in the hot position and increase the number of ones in the cold position and we do not change it in any other position. Hence, it gets us to a population with all positions balanced, thus it has the optimal diversity.

    If we do not generate $\tilde x_i$, then in the valid bit string $x_i'$ we either keep the hot position at the same place, or we move it to any position with a zero-bit in $x_i$, except for the cold position. Therefore, after we replace $x_i$ with a valid bit string, the new hot position will be in $S_0$, and we will have a one-bit in this position. Similarly, the new cold position will be in $S_1$, and we will have a zero-bit in this position. Hence, any bit string with exactly one one-bit in $S_0$ and exactly one zero-bit in $S_1$ is valid.

    If a bit string with $i$ one-bits has at least two one-bits in $S_0$ (and therefore, it has at least two zero-bits in $S_1$), then after replacing $x_i$ with $x_i'$ we have at least four almost balanced positions, which reduces the diversity. 
\end{proof}

\begin{lemma}\label{lem:h1c0-prob}
    Let $x_i$ be an individual with a zero-bit in the cold position and a one-bit in the hot position. Let $\tilde x_i$ be a bit string, which is different from $x_i$ only in the cold and in the hot positions. If either $x_{i - 1}$ or $x_{i + 1}$ is different from $\tilde x_i$ in only one position, then the probability that we replace $x_i$ in one iteration with a different bit string is in
    \begin{align*}
        \left[\frac{\min\{n - i, i\}+ 1}{en^2}\left(1 - O\left(\frac{1}{n}\right)\right), \frac{2}{n}\left(1 + O\left(\frac{1}{n}\right)\right)\right]
    \end{align*}
    and the probability that we reach the optimal diversity (by generating $\tilde x_i$) is at least $\frac{1}{en^2}(1 - O(\frac{1}{n}))$. Otherwise, the probability that we replace $x_i$ is 
    \begin{align*}
        \left[\frac{1}{2en^2}\left(1 - O\left(\frac{1}{n}\right)\right), \frac{13}{n^2}\right]
    \end{align*}
    and the probability to find the optimal population is at least $\frac{1}{en^3}(1 - O(\frac{1}{n}))$.
\end{lemma}

\begin{proof}
    We use the same notation as in Lemma~\ref{lem:h1c1} and call the bit strings which can replace $x_i$ without reducing the diversity \emph{valid}. 
    By Lemma~\ref{lem:h1c0-options}, the set of valid bit strings is the set of bit strings with at most one one-bit in $S_0$ and at most one zero-bit in $S_1$. If we choose as a parent an individual which has $k$ one-bits in $S_0$, then we need to flip at least $k - 1$ of these one-bits to create a bit string which could replace $x_i$. The probability of this is at most
    \begin{align*}
        \frac{k}{n^{k - 1}}\left(1-\frac{1}{n}\right) + \frac{1}{n^k} \le \frac{k}{n^{k - 1}},
    \end{align*}
    which is monotonically decreasing in $k$ for all $k \in [1..n - 1]$.
    Similarly, if we choose an individual with $k$ zero-bits in $S_1$ as a parent, this probability is also at most $\frac{k}{n^{k - 1}}$. 

    All individuals $x_j$ with $j \le i - 3$ have at least $3$ zero-bits in $S_1$, as well as all individuals $x_j$ with $j \ge i + 3$ have at least $3$ one-bits in $S_0$. Hence, if we choose $x_j$ with $|j - i| \ge 3$ as a parent, then the probability to generate a valid bit string is at most $\frac{3}{n^2}$.
    
    If we choose $x_{i \pm 2}$ as a parent, then it either has at least two one-bits in $S_0$ or two zero-bits in $S_1$, hence the probability to create a valid bit string is at most $\frac{2}{n}$.

    If we choose $x_{i - 1}$ as a parent, then we have two options. \textbf{First, if $H(x_{i - 1}, \tilde x_i) = 1$,} then $x_{i - 1}$ has exactly one zero-bit in $S_1$ and no one-bits in $S_0$. Consequently, to create a valid bit string we must flip one of $n - i + 1$ zero-bits. By the union bound over all zero-bits, the probability of this event is at most $\frac{n - i + 1}{n}$. Also, to create a valid bit string it is enough to flip one of $n - i + 1$ zero bits and not to flip any other bits, the probability of which is
    \begin{align*}
        \frac{n - i + 1}{n} \left(1 - \frac{1}{n}\right)^{n - 1} \ge \frac{n - i + 1}{en}.
    \end{align*}
    At the same time to generate $\tilde x_i$, which would give us the optimal diversity, we can flip the only zero-bit in $S_1$ and do not flip any other bit, the probability of which is
    \begin{align*}
        \frac{1}{n}\left(1 - \frac{1}{n}\right)^{n - 1} \ge \frac{1}{en}.
    \end{align*}
    \textbf{Second, if $H(x_{i - 1}, \tilde x_i) > 1$,} then $x_{i - 1}$ has at least two zero-bits in $S_1$. Hence, the probability to create a valid bit string is at most $\frac{2}{n}$.

    If we choose $x_{i + 1}$ as a parent, then we also have two options. \textbf{First, if $H(x_{i + 1}, \tilde x_i) = 1$,} then $x_{i + 1}$ has exactly one one-bit in $S_0$ and no zero-bits in $S_1$. Consequently, to create a valid bit string we must flip one of $i + 1$ one-bits. By the union bound over all one-bits, the probability of this event is at most $\frac{i + 1}{n}$. Also, to create a valid bit string it is enough to flip one of $i + 1$ one-bits and not to flip any other bits, the probability of which is
    \begin{align*}
        \frac{i + 1}{n} \left(1 - \frac{1}{n}\right)^{n - 1} \ge \frac{i + 1}{en}.
    \end{align*}
    At the same time to generate $\tilde x_i$, which would give us the optimal diversity, we can flip the only one-bit in $S_0$ and do not flip any other bit, the probability of which is
    \begin{align*}
        \frac{1}{n}\left(1 - \frac{1}{n}\right)^{n - 1} \ge \frac{1}{en}.
    \end{align*}
    \textbf{Second, if $H(x_{i + 1}, \tilde x_i) > 1$,} then $x_{i + 1}$ has at least two one-bits in $S_0$. Hence, the probability to create a valid bit string is at most $\frac{2}{n}$.

    Finally, if we choose $x_i$ as a parent then we create $\tilde x_i$ with probability 
    \begin{align*}
        \frac{1}{n^2} \left(1 - \frac{1}{n}\right)^{n - 2} \ge \frac{1}{en^2},
    \end{align*}
    since for this we can flip the two bits in the hot and in the cold positions and do not touch any other bits. To create a valid bit string different from $x_i$, we must either flip the bit in the cold position or flip the bit in the hot position. The probability of this is at most $\frac{2}{n}$. To create a valid bit string it is also sufficient to either flip the bit in the hot position and any of $n - i$ zero-bit and do not flip any other bits. We also can flip the bit in the cold position, one of $i$ one-bits and do not flip any other bit to generate a valid bit string. The probability that at least one of these two events occur is at least
    \begin{align*}
        \max\left\{ \frac{n - i}{n^2} \left(1 - \frac{1}{n}\right)^{n - 2}, \frac{i}{n^2} \left(1 - \frac{1}{n}\right)^{n - 2}\right\} \ge 
        \frac{\max\{n - i, i\}}{en^2}.
    \end{align*}

    Now we compute the lower and upper bounds on the probability to replace $x_i$ with a valid bit string. Let $j$ be the index of an individual which we choose as the parent. We distinguish two cases.
    
    \textbf{Case 1:} either $x_{i - 1}$ or $x_{i + 1}$ is in distance $1$ from $\tilde x_i$. Then the probability that we replace $x_i$ is at most
    \begin{align*}
        \Pr[|j - i| \ge 3] &\cdot \frac{3}{n^2} + \Pr[j \in \{i - 2, i, i + 2\}] \cdot \frac{2}{n} + \Pr[j = i \pm 1] \cdot 1 \\
        &\le \frac{n - 4}{n + 1} \cdot \frac{3}{n^2} + \frac{3}{n + 1} \cdot \frac{2}{n} + \frac{2}{n + 1} = \frac{2}{n}\left(1 + O\left(\frac{1}{n}\right)\right).
    \end{align*}
    To compute the lower bound we define $i' = i - 1$, if $x_{i - 1}$ is in distance one from $\tilde x_i$ and $i' = i + 1$ otherwise, hence we have $H(x_{i'}, \tilde x_i) = 1$. Therefore, the probability to generate a valid bit string is at least
    \begin{align*}
        \Pr[j = i'] &\cdot \frac{\min\{n - i + 1, i + 1\}}{en} = \frac{\min\{n - i, i\} + 1}{en(n + 1)} \\
        &= \frac{\min\{n - i, i\} + 1}{en^2}\left(1 - O\left(\frac{1}{n}\right)\right).
    \end{align*}
    The probability to generate $\tilde x_i$ is at least
    \begin{align*}
        \Pr[j = i'] \cdot \frac{1}{en} \ge \frac{1}{en^2}\left(1 - O\left(\frac{1}{n}\right)\right).
    \end{align*} 
    
    \textbf{Case 2:} when both $x_{i - 1}$ and $x_{i + 1}$ are in distance more than one from $\tilde x_i$, then the probability to create a valid bit string is at most
    \begin{align*}
        \Pr[|j - i| \ge 3] &\cdot \frac{3}{n^2} + \Pr[|j - i| \le 2] \cdot \frac{2}{n}\\
        &= \frac{n - 4}{n + 1} \cdot \frac{3}{n^2} + \frac{5}{n + 1} \cdot \frac{2}{n} \le \frac{13}{n^2},
    \end{align*}
    and is at least
    \begin{align*}
        \Pr[j = i] \cdot \frac{1}{2en} \ge \frac{1}{2en^2} \left(1 - O\left(\frac{1}{n}\right)\right).
    \end{align*}
    The probability to create $\tilde x_i$ in this case is at least the probability to create it from $x_i$, that is,
    \begin{align*}
        \Pr[j = i] \cdot \frac{1}{en^2} = \frac{1}{en^3}\left(1 - O\left(\frac{1}{n}\right)\right).
    \end{align*}
\end{proof}

\section{Runtime Analysis}
\label{sec:runtime}

The main result of this section is the following theorem.

\begin{theorem}\label{thm:runtime}
    The expected runtime until the \gsemod finds a population with an optimal diversity on \oneminmax starting from an almost balanced population is $O(n^2)$ iterations.
\end{theorem}

From Section~\ref{sec:replace-prob} we see that the only way to find a population with optimal diversity is to replace individual $x_i$ which has a zero-bit in the cold position and a one-bit in the hot position with $\tilde x_i$, which is different from $x_i$ only in the two almost balanced positions (as defined in Lemma~\ref{lem:h1c0-options}). To show the $O(n^2)$ bound on the runtime of the last optimization step, we aim at proving that we often get to the state of the algorithm when there is a linear number of such individuals which allow us to find the optimal diversity and that we spend enough time in this state before leaving it. The main difficulty in this proof is that we might have individuals which can be replaced with a high probability, which does not allow us to stay in this state for long enough.

To track the population dynamics in the last optimization stage for an almost balanced population $P$ we introduce the following notation.
\begin{itemize}
    
    \item We denote by $I_{01}(P)$ the set of indices $i$ of individuals in $P$ such that $x_i$ has a zero-bit in the hot position and a one-bit in the cold position. By Lemma~\ref{lem:h0c1}, such individuals cannot be replaced.
    
    \item We denote by $I_{00}(P)$ the set of indices $i$ of individuals in $P$ such that $x_i$ has zero-bits in both the cold and the hot positions. By Lemma~\ref{lem:h0c0} replacing these individuals can only move the cold position to another place, but cannot improve the diversity. We also denote by $J_{00}(P)$ the subset of $I_{00}(P)$ of indices $i$ such that $x_i$ is different from $x_{i + 1}$ only in the cold position. By Lemma~\ref{lem:h0c0} the probability to replace such $x_i$ is $\Theta(\frac{i}{n^2})$, which is asymptotically larger than the $O(\frac{1}{n^2})$ probability for the individuals in $I_{00}(P) \setminus J_{00}(P)$ for any $i = \omega(1)$. Informally, this subset should be seen as a subset of indices of individuals which are too easy to replace.
    
    \item We denote by $I_{11}(P)$ the set of indices $i$ of individuals in $P$ such that $x_i$ has one-bits in both the cold and the hot positions. By Lemma~\ref{lem:h1c1} replacing these individuals can only move the hot position to another place, but cannot improve the diversity. We also denote by $J_{11}(P)$ the subset of $I_{11}$ of indices $i$ such that $x_i$ is different from $x_{i - 1}$ only in the hot position. By Lemma~\ref{lem:h1c1} the probability to replace such $x_i$ is $\Theta(\frac{n - i}{n^2})$,  which is asymptotically larger than the $O(\frac{1}{n^2})$ probability for the individuals in $I_{11}(P) \setminus J_{11}(P)$ for any $(n - i) = \omega(1)$. Informally, this subset should be seen as a subset of indices of individuals which are too easy to replace.
    
    \item We denote by $I_{10}(P)$ the set of indices $i$ of individuals in $P$ such that $x_i$ has a one-bit in the hot position and a zero-bit in the cold position. By $J_{10}(P)$ we denote the subset of $I_{10}(P)$ of indices $i$ such that with $S_0$ and $S_1$ defined as in Lemma~\ref{lem:h1c0-options} either $x_{i - 1}$ has exactly one zero-bit in $S_1$ or $x_{i + 1}$ has exactly one one-bit in $S_0$. By Lemma~\ref{lem:h1c0-prob} the probability to replace such $x_i$ is $O(\frac{1}{n})$ and the probability to find the optimal population by replacing it is $\Omega(\frac{1}{n^2})$. Similar to $J_{00}(P)$ and $J_{11}(P)$, this subset should be seen as a subset of indices of individuals which are easy to replace, but they give us a good chance of finding the optimal population.
\end{itemize}

With the introduced notation we summarize the results of Lemmas~\ref{lem:h0c1}-\ref{lem:h1c0-prob} in Table~\ref{tbl:probs}. We also show the relation between the set sizes in the following lemma.

\begin{lemma}
    \label{lem:set-sizes}
    For any almost balanced population $P$ we have $|I_{00}(P)| = |I_{11}(P)|$ and $|I_{10}(P)| = |I_{01}(P)| + 2$. We also have $|I_{00}(P)| \le \frac{n - 1}{2}$ and $|I_{11}(P)| \le \frac{n - 1}{2}$, and $2 \le I_{10}(P) \le \frac{n + 3}{2}$.
\end{lemma}

\begin{proof}
    The number of individuals with a one-bit in the hot position is 
    \begin{align}\label{eq:one-bits-hot-pos}
        |I_{10}(P)| + |I_{11}(P)| = \frac{n + 3}{2}
    \end{align}
    and the number of individuals with a zero bit in the cold position is
    \begin{align*}
        |I_{10}(P)| + |I_{00}(P)| = \frac{n + 3}{2}.
    \end{align*}
    Hence, $|I_{00}(P)| = |I_{11}(P)|$. The number of individuals with a zero-bit in the hot position is 
    \begin{align*}
        |I_{01}(P)| + |I_{00}(P)| = \frac{n - 1}{2}.
    \end{align*}
    Subtracting this equation from~\eqref{eq:one-bits-hot-pos} we obtain
    \begin{align*}
        |I_{10}(P)| - |I_{01}(P)| = 2.
    \end{align*} 
    This implies $|I_{10}(P)| \ge 2$. We also have
    \begin{align*}
        |I_{00}(P)| = |I_{11}(P)|&= \frac{|I_{00}(P)| + |I_{11}(P)|}{2} \le \frac{|P| - |I_{10}(P)|}{2} \\
        &\le \frac{n + 1 - 2}{2} = \frac{n - 1}{2}.
    \end{align*}
    Finally, we have $I_{10}(P) \le \frac{n + 3}{2}$, since there are at most $\frac{n + 3}{2}$ individuals with a one-bit in the hot position.
\end{proof}

\begin{table*}
    \caption{Summary of the bounds on the probabilities proved in Lemmas~\ref{lem:h0c1}-\ref{lem:h0c0} and~\ref{lem:h1c0-prob}. By $i$ we denote the number of one-bits in the individuals for which this probability is estimated.}
    \label{tbl:probs}
    \begin{center}
        \begin{tabular}[t]{lccc}
            \toprule
            Set & $\Pr[\text{replace}] \ge$ & $\Pr[\text{replace}] \le$ & $\Pr[\text{create } \tilde x] \ge$ \\ \midrule
            $I_{01}(P)$ & 0 & 0 & 0 \\
            $I_{00}(P) \setminus J_{00}(P)$ & $\frac{i}{en^3}(1 - O(\frac{1}{n}))$ & $\frac{7}{en^2}(1 + O(\frac{1}{n}))$ & 0 \\[3pt]
            $J_{00}(P)$ & $\frac{i}{en^2}(1 - O(\frac{1}{n}))$ & $\frac{i + 2}{en^2}(1 + O(\frac{1}{n}))$ & 0 \\[3pt]
            $I_{11}(P) \setminus J_{11}(P)$ & $\frac{n - i}{en^3}(1 - O(\frac{1}{n}))$ & $\frac{7}{en^2}(1 + O(\frac{1}{n}))$ & 0 \\[3pt]
            $J_{11}(P)$ & $\frac{n - i}{en^2}(1 - O(\frac{1}{n}))$ & $\frac{n - i + 2}{en^2}(1 + O(\frac{1}{n}))$ & 0 \\[3pt]
            $I_{10}(P) \setminus J_{10}(P)$ & $\frac{1}{2en^2}(1 - O(\frac{1}{n}))$ & $\frac{13}{n^2}$ & $\frac{1 - O(\frac{1}{n})}{en^3}$ \\[3pt]
            $J_{10}(P)$ & $\frac{\min\{n - i, i\} + 1}{en^2}(1 - O(\frac{1}{n}))$ & $\frac{2}{n}(1 + O(\frac{1}{n}))$ & $\frac{1 - O(\frac{1}{n})}{en^2}$ \\ \bottomrule
        \end{tabular}
    \end{center}
\end{table*}

For any almost balanced population $P$ we additionally define $\Jhot(P)$ as the set of indices $i$ such that $x_i$ is different from $x_{i - 1}$ only in the hot position. The following lemma describes the relation between $\Jhot(P)$ and the previously introduced notation.

\begin{lemma}
    \label{lem:jhot}
    Let $P$ be an almost balanced population. Then $\Jhot(P)$ is a superset of $J_{11}(P)$ and it is a subset of $J_{11}(P) \cup J_{10}(P)$.
\end{lemma}

\begin{proof}
    By the definition of $J_{11}(P)$ for all $i \in J_{11}(P)$ we have that $x_i$ is different from $x_{i - 1}$ only in the hot position, thus $i \in \Jhot(P)$. Hence, $\Jhot(P) \supset J_{11}(P)$.
    
    Consider $i \in \Jhot(P)$. There are two possible values of the bit in the cold position of $x_i$. If this bit is a one-bit, then $x_i$ has one-bits in both almost balanced positions and therefore $i$ belongs to $J_{11}(P)$. Otherwise, if we have a zero-bit in the cold position, then $i$ is in $I_{10}(P)$ and $\tilde x_i$ (as it was defined in Lemma~\ref{lem:h1c0-options}) is different from $x_{i - 1}$ only in the cold position. Hence, $i \in J_{10}(P)$.
\end{proof}

We aim at showing that the algorithm is likely to obtain $\Omega(n)$ individuals with index in $I_{10}(P)$ and it stays in this state for long enough to have a good probability to generate an individual which yields the optimal diversity. The individuals with index in $J_{00}(P)$, $J_{11}(P)$ and $J_{10}(P)$, however, increase the probability to change population $P$, which reduces the time which we are expected to spend with $|I_{10}(P)| = \Omega(n)$. The individuals with index in $J_{10}(P)$ also play an opposite (positive) role, since they increase the probability to improve the diversity. Hence the good state of the algorithm, from which we have a good probability to find the optimal diversity and in which we have a not too large probability to leave this state, is when we have small $J_{00}(P)$ and small $J_{11}(P)$ and we have large $I_{10}(P)$. These observations lead us to distinguishing the following states of the algorithm with an almost balanced population $P$.

    \textbf{State 3:} we have $|I_{10}(P)| \ge \frac{n}{32}$, $|\Jhot(P)| \le 19$ and $|J_{00}(P)| \le 9$.
    
    \textbf{State 2:} we are not in State 3 and have $|\Jhot(P)| \le 17$. 
    
    \textbf{State 1:} all other possible situations. 

For $i = 1, 2, 3$ we say that ``the algorithm is in State $i$ at iteration $t$'' or just ``population $P_t$ is in State $i$'' for brevity, if $P_t$ satisfies the conditions of State $i$.

We also split the algorithm run into phases. We later show that each such phase ends in $O(n)$ iterations in expectation and with probability at least $\Omega(\frac{1}{n})$ it ends in the population with optimal diversity. The first phase starts at the first iteration. The current phase ends in the end of iteration $t$ and the a new phase starts at the beginning of iteration $t + 1$ if populations $P_t$ and $P_{t + 1}$ in these iterations are such one of the follwoing conditions is satisfied.
\begin{enumerate}
    \item $P_t$ is in State 1 and $P_{t + 1}$ has an optimal diversity.
    \item $P_t$ is in State 2 and $P_{t + 1}$ is not in State 3 and we either have $|\Jhot(P_{t + 1})| > |\Jhot(P_t)|$ or in iteration $t$ we replace an individual with index $I_{10}(P_t)$ with a different bit string (and probably, $P_{t + 1}$ has an optimal diversity)
    \item $P_t$ is in State 3 and $P_{t + 1}$ is different from $P_t$ (that is, any change in State 3, including finding the optimal diversity, ends the current phase).
\end{enumerate}

From this definition of phases it follows that in the frames of one phase we cannot go from State 2 to State 1 (since increasing $|\Jhot(P_t)|$ ends the phase or moves us to State 3) and we also cannot go from State 3 to any other state (since any change in the population ends the phase). We illustrate the possible transitions during one phase in Figure~\ref{fig:phase}. Note that although this illustration resembles a Markov chain, it is not one, since the transition probabilities can vary in time.

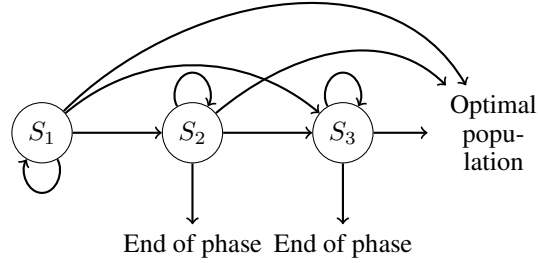
\begin{figure}
    \begin{center}
        \begin{tikzpicture}
            \node (s1) [draw, circle] at (0, 0) {$S_1$};
            \node (s2) [draw, circle] at (2, 0) {$S_2$};
            \node (s3) [draw, circle] at (4, 0) {$S_3$};

            \node (opt) [text width = 1.5cm, align = center] at (6, 0) {Optimal population};
            % \draw (9, 0) circle (3mm);

            \node (eop1) at (2, -1.5) {End of phase};
            \node (eop2) at (4, -1.5) {End of phase};

            \draw [->, thick] (s1) to[out=45, in=120] (opt);
            \draw [->, thick] (s1) to[out=40, in=140] (s3);
            \draw [->, thick] (s1) to[out=300, in=0] (0, -0.8) to[out=180, in=240] (s1);
            \draw [->, thick] (s1) -- (s2);
            \draw [->, thick] (s2) to[out=40, in=135] (opt);
            \draw [->, thick] (s2) -- (s3);
            \draw [->, thick] (s2) -- (eop1);
            \draw [->, thick] (s2) to[out=120, in=180] (2, 0.8) to[out=0, in=60] (s2);
            \draw [->, thick] (s3) -- (opt);
            \draw [->, thick] (s3) -- (eop2);
            \draw [->, thick] (s3) to[out=120, in=180] (4, 0.8) to[out=0, in=60] (s3);

        \end{tikzpicture}
    \end{center}
    \caption{Illustration of the possible transitions during one phase. This is not a Markov chain, since the transition probabilities are dependent on multiple factors and can be different in different moments of time.}
    \label{fig:phase}
\end{figure}

\subsection{Analysis of State 3}
\label{sec:state3}

In this section we estimate the time we spend in State $3$, show the possible outcomes and their probabilities. By the definition of State~$3$, any change in the population results in the end of the phase, hence once we are in State $3$, we cannot go to States $1$ and $2$ in the same phase, but we can finish the current phase in a population with optimal diversity. The following lemma estimates the expected time until a change in the population happens and the probability that at that change we find the optimal population.

\begin{lemma}\label{lem:state3}
    Let the population $P_\tau$ of \gsemod at some iteration $\tau$ be in State $3$. Then the expected runtime until the end of the phase is at most $64en + O(1)$ and the probability that we find the optimal population in the end of the phase is at least $\frac{1 - O(\frac{1}{n})}{1156n} = \Omega(\frac{1}{n})$.
\end{lemma}
\begin{proof}
    The proof is based on Lemma~\ref{lem:phase}. W.l.o.g. we assume that we start with iteration $\tau = 1$. Then the iterations starting from $\tau$ can be considered as a sequence of trials $\{X_t\}_{t \in \N}$. For each $t$ the outcomes of $X_t$ are $\omega_1$, when $P_{t + 1}$ has an optimal diversity, $\omega_2$, when $P_{t + 1} \ne P_t$, but its diversity is not optimal, and $\omega_3$, when $P_{t + 1} = P_t$. The phase ends at iteration $T$, which is the first $t$ for which we have $X_T \ne \omega_3$ and the probability that $P_{T + 1}$ has an optimal diversity is the probability that $X_T = \omega_1$.
    
    For all $t \in \N$ we define events $A_t$, $B_t$ and $C_t$ similar to Lemma~\ref{lem:phase}, that is, $A_t$ is when $X_t = \omega_1$, $B_t$ is when $X_t = \omega_2$ and $C_t$ is when all $X_\tau$ with $\tau < t$ are $\omega_3$ (note that $C_t$ is undefined for $t = 1$). 
    Events $A_t$ and $B_t$ depend only on the population $P_t$ in the start of iteration $t$. Event $C_t$ implies that in iteration $t \ge 2$ we have the same population $P_t$ as $P_1$, that is, $P_t$ is in State 3.
    Therefore, the probabilities $p_t$ and $q_t$ of events $A_t$ and $B_t$ conditional on $C_t$ (or unconditional for $t = 1$) are the probabilities of these events conditional on $P_t$ being in State 3.

    We now estimate $p_t + q_t$ and $\frac{q_t}{p_t}$. By Lemma~\ref{lem:h1c0-prob} the probability that we change an individual with index in $I_{10}(P_t)$ is at least $\frac{1}{2en^2}(1 - O(\frac{1}{n}))$. Conditional on $C_t$, we have $|I_{10}(P_t)| \ge \frac{n}{32}$ (since $P_t$ belongs to State 3). Hence, the probability of $A_t \cup B_t$ in this case is at least the probability that we  change one of these individuals, that is,
    \begin{align*}
        p_t + q_t \ge \frac{n}{32} \cdot \frac{1}{2en^2}\left(1 - O\left(\frac{1}{n}\right)\right) = \frac{1}{64en}\left(1 - O\left(\frac{1}{n}\right)\right) \eqqcolon p.
    \end{align*}

    To show an upper bound on $\frac{q_t}{p_t}$ we use
    \begin{align*}
        \frac{q_t}{p_t} \le \frac{q_t + p_t}{p_t} = \frac{\Pr[A_t \cup B_t \mid C_t]}{\Pr[A_t \mid C_t]}.
    \end{align*}

    These probabilities depend only on the population $P_t$, in which by condition $C_t$ we have $|J_{00}(P_t)| \le 9$ and $|J_{11}(P_t)| \le |\Jhot(P_t)| \le 19$ and $|I_{10}(P_t)| \ge \frac{n}{32}$. We denote $m \coloneqq |I_{10}(P_t)|$ and $m'\coloneqq |J_{10}(P_t)|$. By Lemma~\ref{lem:h1c0-prob}, for each individual with index in $J_{10}(P_t)$ the probability to replace it and get the optimal diversity is at least $\frac{1}{en^2}(1 - O(\frac{1}{n}))$ and for each individual in $I_{10}(P_t) \setminus J_{10}(P_t)$ this probability is at least $\frac{1}{en^3}(1 - O(\frac{1}{n}))$. Hence, 
    \begin{align}
        \label{eq:p_t}
        \begin{split}
            p_t &= \Pr[A_t \mid C_t] \ge \frac{m - m'}{en^3}\left(1 - O\left(\frac{1}{n}\right)\right) + \frac{m'}{en^2}\left(1 - O\left(\frac{1}{n}\right)\right) \\&= \left(\frac{m - m'}{en^3} + \frac{m'}{en^2}\right)\left(1 - O\left(\frac{1}{n}\right)\right) \\
            &= \left(\frac{m}{en^3} + \frac{m'}{en^2}\right)\left(1 - O\left(\frac{1}{n}\right)\right) = \frac{1 - O\left(\frac{1}{n}\right)}{en^3} \cdot (nm' + m),
        \end{split}
    \end{align}
    where we hid the $(-\frac{m'}{en^3})$ term inside the $O(\frac{1}{n})$ term.

    Event $A_t \cup B_t$ occurs when (and only when) we replace an individual with a different one. Let $D_t^i$ be the event that we replace individual $i$. Since these events are disjoint, we have \begin{align*}
        \Pr[A_t \cup B_t \mid C_t] = \sum_{i = 0}^n \Pr[D_t^i \mid C_t].
    \end{align*}
    By using our estimates to replace an individual based on its index $i$ obtained in Section~\ref{sec:replace-prob} and summarized in Table~\ref{tbl:probs}, and also the conditions on $P_t$ implied by State 3, we have
    \begin{align*}
        \Pr&[A_t \cup B_t \mid C_t] \le \sum_{i \in J_{10}(P_t)} \frac{2}{n}\left(1 + O\left(\frac{1}{n}\right)\right) + \sum_{i \in I_{10}(P_t) \setminus J_{10}} \frac{13}{n^2}\\
        &+ \sum_{i \in J_{00}(P_t)} \frac{i + 2}{en^2}\left(1 + O\left(\frac{1}{n}\right)\right) + \sum_{i \in I_{00}(P_t) \setminus J_{00}} \frac{7}{en^2}\left(1 + O\left(\frac{1}{n}\right)\right) \\
        &+ \sum_{i \in J_{11}(P_t)} \frac{n - i + 2}{en^2}\left(1 + O\left(\frac{1}{n}\right)\right) + \sum_{i \in I_{11}(P_t) \setminus J_{11}} \frac{7}{en^2}\left(1 + O\left(\frac{1}{n}\right)\right) \\
        &\le |J_{10}(P_t)| \cdot \frac{2}{n}\left(1 + O\left(\frac{1}{n}\right)\right) + |I_{10}(P_t)| \cdot \frac{13}{n^2} \\
        &+ |J_{00}(P_t) \cup J_{11}(P_t)| \cdot \frac{1}{en}\left(1 + O\left(\frac{1}{n}\right)\right) \\
        &+ |I_{00}(P_t) \cup I_{11}(P_t)| \cdot \frac{7}{en^2} \left(1 + O\left(\frac{1}{n}\right)\right) \\
        &\le \left(\frac{2m'}{n} + \frac{13m}{n^2} + \frac{19 + 9}{en} + \frac{7n}{en^2} \right)\left(1 + O\left(\frac{1}{n}\right)\right) \\
        &\le \frac{1 + O\left(\frac{1}{n}\right)}{en^2} \cdot \left(2em'n + 13em + 35n\right).
    \end{align*}

    By~\eqref{eq:p_t}, we obtain
    \begin{align*}
        \frac{q_t}{p_t} &\le \frac{\frac{1 + O\left(\frac{1}{n}\right)}{en^2} \cdot \left(2em'n + 13em + 35n\right)}{\frac{1 - O\left(\frac{1}{n}\right)}{en^3} \cdot (nm' + m)} \\
        &= (n + O(1)) \cdot \frac{2em'n + 13em + 35n}{nm' + m} \le (n + O(1)) \left(13e + \frac{35n}{m}\right) \\
        &\le (n + O(1)) \left(13e + 35 \cdot 32\right) \le 1156n +O(1).
    \end{align*} 

    Therefore, by Lemma~\ref{lem:phase} we have $E[T] \le \frac{1}{p} \le 64en + O(1)$ and the probability that the phase ends in the optimal population, that is, $X_T = \omega_1$, is at least $\frac{1}{1 + 1156n + O(1)} = \frac{1 - O(\frac{1}{n})}{1156n} = \Omega(\frac{1}{n})$. \qedhere

\end{proof}

\subsection{Analysis of State 2}
\label{sec:state2}

We proceed with considering the possible scenarios when we are in State 2. To ease the reading we introduce the following notation. For an almost balanced population $P$ let $I_{H1}(P)$ be the set of indices of individuals in $P$ with a one-bit in the hot position and $I_{H0}(P)$ be a set of indices of individuals with a zero-bit in the hot position. Let also $I_{C1}(P)$ and $I_{C0}(P)$ be the sets of indices of individuals with a one-bit and a zero-bit in the cold position correspondingly. 
% \todo{maybe we should introduce it earlier, together with other $I$ and $J$ sets.}

For an almost balanced population $P$ we call a balanced position $i$ a \emph{cold-candidate} position, if at least $\frac{n}{16}$ individuals in $P$ with index in $I_{C0}(P)$ have a one-bit in position $i$. The following lemma shows that when $|I_{10}(P)| \le \frac{n}{32}$, then moving the cold position to one of such positions would bring us to State 3 and it is also relatively easy to do.

\begin{lemma}\label{lem:cold-candidates-pos}
    Consider iteration $t$ and assume that $P_t$ is in State 2 and that $|I_{10}(P_t)| \le \frac{n}{32}$. For all $i \in [1..n]$ if position $i$ is a cold-candidate, then the probability to move the cold position to $i$ without moving the hot position in one iteration is at least $\frac{1}{16en^2}$ and after this move we have $|I_{10}(P_{t + 1})| \ge \frac{n}{32}$.
\end{lemma}
\begin{proof}
    To move the cold position to $i$, we can choose one of the $\frac{n}{16}$ individuals which have a one in position $i$ and zero in the cold position and flip these two bits (and only them) in it. The new individual has the same number of one-bits as its parent (and thus, it is on the same fitness level) and it also makes the cold position balanced and adds an extra zero-bit to position $i$ if it replaces its parent in the population. Thus this replacement does not reduce the diversity and will be accepted by the \gsemod. The hot position stays at the same place in this case.

    The probability to choose one of these individuals is $\frac{n}{16(n + 1)}$ and the probability to flip two particular bits and do not flip the others is $\frac{1}{n^2}(1 - \frac{1}{n})^{n - 2}$. Therefore, the probability that we move the cold position to $i$ and do not move the hot position is at least
    \begin{align*}
        \frac{n}{16(n + 1)} \cdot \frac{1}{n^2}\left(1 - \frac{1}{n}\right)^{n - 2} = \frac{1}{16(n^2 - 1)}\left(1 - \frac{1}{n}\right)^{n - 1} \ge \frac{1}{16en^2}.
    \end{align*}

    Since position $i$ is balanced and there are at least $\frac{n}{16}$ individuals with index in $I_{C0}(P_t)$ which have a one-bit in position $i$, then at most $\frac{n + 1}{2} - \frac{n}{16}$ individuals with index in $I_{C1}(P_t)$ have a one-bit in this position. Since there are $\frac{n - 1}{2}$ one-bits in the cold position, we have $|I_{C1}(P_t)| = \frac{n - 1}{2}$. Therefore, the number of individuals with index in $I_{C1}(P_t)$ and a \emph{zero-bit} in position $i$ is at least $\frac{n - 1}{2} - (\frac{n + 1}{2} - \frac{n}{16}) = \frac{n}{16} - 1$. Some of these individuals might have a zero-bit in the hot position, and then their index is in $I_{01}(P_t)$. By the lemma conditions and by Lemma~\ref{lem:set-sizes} we have $|I_{01}(P_t)| = |I_{10}(P_t)| - 2 \le \frac{n}{32} - 2$. Hence at least
    \begin{align*}
        \left(\frac{n}{126} - 1\right) - \left(\frac{n}{32} - 2\right) \ge \frac{n}{32}
    \end{align*}
    individuals have a one-bit in the hot position and a zero-bit in position $i$. Since we cannot move the cold position to position $i$ by replacing any of these individuals (we need to reduce the number of zero-bits in position $i$, thus the replaced individual must have a one-bit in position $i$), they all are included into  $I_{10}(P_{t + 1})$ after moving the cold position. Therefore, this move gives us a population $P_{t + 1}$ with $|I_{10}(P_{t + 1})| \ge \frac{n}{32}$.
\end{proof}

The following lemma shows that there is always a linear number of cold-candidates in any almost balanced population, independently of the current state.

\begin{lemma}
    \label{lem:cold-candidates-number}
    If $n \ge 5$, then in any almost balanced population there are at least $\frac{n}{8}$ cold-candidate positions.
\end{lemma}
\begin{proof}
    Consider an arbitrary almost balanced population $P$. There are $\frac{n + 3}{2}$ individuals with index in $I_{C0}(P)$, and since they all have different fitness and therefore have a different number of one-bits in them, the total number of one-bits in these individuals is at least
    \begin{align*}
        \sum_{i = 0}^{\frac{n + 1}{2}} i = \frac{(n + 1)(n + 3)}{8}.
    \end{align*}
    At most $\frac{n + 3}{2}$ of these ones are in the hot position and none of them are in the cold position. Hence, these individuals have at least
    \begin{align*}
        \frac{(n + 1)(n + 3)}{8} - \frac{n + 3}{2} = \frac{(n + 3)(n + 1 - 4)}{8} = \frac{n^2 - 9}{8}
    \end{align*}
    one-bits in balanced positions. Assume that there are $s < \frac{n}{8}$ cold-candidate positions. Since these positions are balanced, for each of these positions there are at most $\frac{n + 1}{2}$ individuals with index in $I_{C0}(P)$ and a one-bit in this position. For the other $(n - 2 - s)$ positions there are at most at most $\frac{n}{16} - 1$ individuals with index in $I_{C0}(P)$ and a one-bit in it (otherwise they would be cold-candidates). Therefore, by our assumption on $s$, the number of one-bits in individuals with index in $I_{C0}(P)$ is at most
    \begin{align*}
            s\cdot \frac{n + 1}{2} &+ (n - 2 - s) \cdot \frac{n - 16}{16} = s \cdot \frac{7n + 24}{16} + \frac{(n - 2)(n - 16)}{16} \\
            &\le \frac{\left(\frac{n}{8} - 1\right)(7n + 24) + (n - 2)(n - 16)}{16} \\
            &= \frac{\frac{15}{16}n^2 - 3n + 4}{8} = \frac{n^2 - 9}{8} - \frac{\frac{n^2}{16} + 3n - 13}{8}.
    \end{align*}
    The last term is strictly positive, when $n \ge 5$, hence this is strictly less than $\frac{n^2 - 9}{8}$, which is the lower bound on the number of zero-bits in these individuals. Hence, our assumption is wrong, and there are at least $\frac{n}{8}$ cold-candidates. \qedhere
    
\end{proof}

With these two lemmas we are in position to show the lower bound on the probability that $P_{t + 1}$ is in State 3 when we have $P_t$ in State 2 and $|I_{10}(P_t)| \le \frac{n}{32}$. Later we will show that with larger $|I_{10}(P_t)|$ we can rely on finding the optimal population without going to State 3. 

\begin{lemma}\label{lem:state2-to-state3}
    For any iteration $t$ with $P_t$ in State 2, and $|I_{10}(P_t)| \le \frac{n}{32}$ the probability that $P_{t + 1}$ is in State 3 is at least $\frac{1}{256en}$.
\end{lemma}
\begin{proof}    
    By Lemma~\ref{lem:cold-candidates-number} we have at least $\frac{n}{8}$ cold-candidates positions, and since by the lemma conditions we have $|I_{10}(P_t)| \le \frac{n}{32}$, then by Lemma~\ref{lem:cold-candidates-pos} the probability to move the cold positions to each of these positions is at least $\frac{1}{16en^2}$ and it yields $|I_{10}(P_{t + 1})| \ge \frac{n}{32}$. 

    We now show that the move of the cold position (without moving the hot position) cannot make $|\Jhot(P_{t + 1})| > |\Jhot(P_t)| + 2$. Assume that we have moved the cold position by replacing individual $x_i$. Then for all $j \notin \{i, i + 1\}$ we have $j \in \Jhot(P_{t + 1})$ if and only if $j \in \Jhot(P_t)$, since the hot position is at the same place in $P_t$ and $P_{t + 1}$ and both $x_j$ and $x_{j - 1}$ have not changed in iteration $t$.
    Therefore, the move of the cold position can only add two indices to $\Jhot(P_{t + 1})$ (compared to $\Jhot(P_t)$), that are, $i$ and $i + 1$. Since in State~$2$ we have $|\Jhot(P_t)| \le 17$, after the move of the cold position to any of the cold-candidate positions, we will have $|\Jhot(P_{t + 1})| \le 19$.

    Hence, moving the cold position to any of the cold-candidates satisfies $|I_{10}(P_{t + 1})| \ge \frac{n}{32}$ and $|\Jhot(P_{t + 1})| \le 19$. To get $P_{t + 1}$ in State 3 we also need to make $|J_{00}(P_{t + 1})| \le 9$, and now we show that at least $\frac{n}{16}$ cold-candidates allow us to do that.
    
    Consider an event when we move the cold position to position $j$ (which is not necessarily a cold-candidate) by replacing individual $x_i$ with $x_i'$. Consider also some index $k \in I_{H0}(P_t) \supset J_{00}(P_t)$. Note that $I_{H0}(P_{t + 1})$ is equal to $I_{H0}(P_t)$, since we do not move the hot position and we do not change the bit value in it in any individual, including $x_i$. Hence, $J_{00}(P_{t + 1}) \subset I_{H0}(P_t)$. There are two cases, in which $k$ belongs to $J_{00}(P_{t + 1})$. The first case is if $k = i$, then $x_i'$ can be different from $x_{i + 1}$ only in position $j$, then it will be included into $J_{00}(P_{t + 1})$ by the definition of this set. In the second case, if $k \ne i$ and $k \ne i - 1$, then $k \in J_{00}(P_{t + 1})$ only if $x_k$ was different from $x_{k - 1}$ only in position $j$ before the iteration. Note that $k = i - 1$ can never be added to $J_{00}(P_{t + 1})$, since $x_i'$ must have a zero-bit in position $j$ (moving the cold position to $j$ implies that we increase the number of zero-bits in it), and thus $x_i'$ cannot be different from $x_{i - 1}$ only in the new cold position. 

    With this observation, for each position $j \in [1..n]$, except for the hot and the cold positions, we denote by $J_j$ the set of indices $k \in I_{H0}(P_t)$ such that in $P_t$ individual $x_k$ is different from $x_{k + 1}$ only in position $j$. Note that these are disjoint sets for different positions $j$ and they are independent of the index $i$ of the individual we replace. Then for all $j \in [1..n]$ (except the two almost balanced positions) moving the hot position to $j$ will give us $|J_{00}(P_{t + 1})| \le |J_j| + 1$, where we add one to take into account the case when $k = i$, that is, when the index of the changed individual is added to $J_{00}(P_{t + 1})$.

    Since $J_j$ are disjoint subsets of $I_{H0}(P_t)$, then we have $\sum_j |J_j| \le \frac{n - 1}{2}$. If we assume that there are at least $\frac{n}{16}$ positions $j$ that have $|J_j| > 8$, then we have $\sum_j |J_j| > \frac{n}{16} \cdot 8 = \frac{n}{2} > \frac{n - 1}{2}$, hence we must have less than $\frac{n}{16}$ such positions. Consequently, there are at least $\frac{n}{8} - \frac{n}{16} = \frac{n}{16}$ cold-candidates $j$ with $|J_j| \le 8$, moving the cold position to which yields $|J_{00}(P_{t + 1})| \le 8 + 1 = 9$. The probability to move the cold position to any of these $\frac{n}{16}$ cold-candidates is at least 
    \begin{align*}
        \frac{n}{16} \cdot \frac{1}{16en^2} = \frac{1}{256en}. &\qedhere
    \end{align*}
\end{proof}

To bound the probability of finding the optimal population and the probability of ending the phase when we are in State 2 we also need the following auxiliary result.

\begin{lemma}\label{lem:end-phase-state2}
    For any iteration $t$ with $P_t$ in State 2, the probability that either we have $|\Jhot(P_{t + 1})| > |\Jhot(P_t)|$ or in iteration $t$ we change an individual with index in $I_{10}(P_t)$ is at most 
    \begin{align*}
        \frac{15 + 2|J_{10}(P)|}{en} \left(1 + O\left(\frac{1}{n}\right)\right).
    \end{align*}
    The probability that we find the optimal population in one iteration is at least
    \begin{align*}
        \left(\frac{1}{en^2}|J_{10}(P_t)| + \frac{1}{en^3}|I_{10}(P_t) \setminus J_{10}(P_t)|\right) \left(1 - O\left(\frac{1}{n}\right)\right).
    \end{align*}
\end{lemma}
\begin{proof}
    The lower bound on the probability to find the optimal population follows from Lemma~\ref{lem:h1c0-prob}. Since for any individual with index in $J_{10}(P_t)$ the probability to replace it and obtain the optimal diversity is at least $\frac{1 - O(\frac{1}{n})}{en^2}$ and for any individual with index in $I_{10}(P) \setminus J_{10}(P_t)$ this probability is at least $\frac{1 - O(\frac{1}{n})}{en^3}$ and since these are disjoint events for different individuals, we have that this probability is at least
    \begin{align*}
        \left(\frac{1}{en^2}|J_{10}(P_t)| + \frac{1}{en^3}|I_{10}(P_t) \setminus J_{10}(P_t)|\right) \left(1 - O\left(\frac{1}{n}\right)\right).
    \end{align*}

    Let $A$ be the event when either we replace an individual with index in $|I_{10}(P_t)|$ or we get $|\Jhot(P_{t + 1})| > |\Jhot(P_t)|$. Event $A$ it can be represented as a union of three disjoint events $A = A_{00} \cup A_{11} \cup A_{10}$, where $A_{10}$ is the event when we replace an individual with index in $I_{10}(P_t)$ and $A_{00}, A_{11}$ are the events when we replace an individual with index in $I_{00}(P_t)$ or $I_{11}(P_t)$ correspondingly and get $|\Jhot(P_{t + 1})| > |\Jhot(P_t)|$. Note that by Lemma~\ref{lem:h0c1} we cannot replace an individual with index in $I_{01}(P_t)$, hence we do not consider this as a part of event $A$. Therefore, we have $\Pr[A] = \Pr[A_{00}] + \Pr[A_{11}] + \Pr[A_{10}]$. We estimate each of the three probabilities separately.

    Replacing an individual with index in $I_{11}(P_t)$ is a super-event of $A_{11}$, hence its probability is not smaller than $\Pr[A_{11}]$.
    By Lemma~\ref{lem:h1c1}, the probability to replace an individual with index in $I_{11}(P_t)$ is at most $\frac{1}{en}(1 + O(\frac{1}{n}))$, if its index is in $J_{11}(P_t)$ and it is at most $\frac{7}{en^2}(1 + O(\frac{1}{n}))$ otherwise. Since in State 2 we have $|J_{11}(P_t)| \le |\Jhot(P_t)| \le 17$ and since by Lemma~\ref{lem:set-sizes} we always have $|I_{11}(P_t)| \le \frac{n - 1}{2}$, then we have
    \begin{align*}
        \Pr[A_{11}] &\le \frac{(n - 1)}{2} \cdot \frac{7}{en^2}\left(1 + O\left(\frac{1}{n}\right)\right) + 17 \cdot \frac{1}{en}\left(1 + O\left(\frac{1}{n}\right)\right) \\
        &\le \frac{41}{2en}\left(1 + O\left(\frac{1}{n}\right)\right).
    \end{align*}

    Similarly, by Lemma~\ref{lem:h1c0-prob}, the probability to replace an individual with index in $I_{10}(P_t)$ is at most $\frac{13}{n^2}$, if its index is not in $J_{10}(P_t)$ and it is at most $\frac{2}{n}$ otherwise. Since $|I_{10}(P_t)| \le \frac{n + 3}{2}$, the probability of $A_{10}$ is at most
    \begin{align*}
        \Pr[A_{10}] &\le |I_{10}(P_t) \setminus J_{10}(P_t)| \cdot \frac{13}{n^2} + |J_{10}(P_t)| \cdot \frac{2}{n}\left(1 + O\left(\frac{1}{n}\right)\right) \\
        &\le \frac{13(n + 3)}{2n^2} + |J_{10}(P_t)| \cdot \frac{2}{n}\left(1 + O\left(\frac{1}{n}\right)\right) \\
        &= \left(\frac{13}{2n} + |J_{10}(P_t)| \cdot \frac{2}{n}\right)\left(1 + O\left(\frac{1}{n}\right)\right).
    \end{align*}

    To estimate the probability of $A_{00}$, we use a more rigorous approach. First we note, that by replacing an individual with index in $I_{00}(P_t)$ we do not move the hot position by Lemma~\ref{lem:h0c0}. Hence, similar to our argument in Lemma~\ref{lem:state2-to-state3}, the only two indices which can be added to $\Jhot(P_{t + 1})$ which are not in $\Jhot(P_t)$ are the index of the changed individual $i$ or the index of its neighbour $i + 1$. However, there is a guarantee that $i$ in this case cannot be included into $\Jhot(P_{t + 1})$, since it is in $I_{00}(P_t)$ and even after replacement new $x_i$ has a zero-bit in the hot position. Hence, it will be different from $x_{i - 1}$ not only in the hot position and therefore, it is not in $\Jhot(P_{t + 1})$.

    To replace $x_i$ with a bit string which is different from $x_{i + 1}$ only in the hot position we must generate a particular bit string $x_i'$ in level $i$. For all $j \ne i$ the Hamming distance from this bit string to $x_j$ is at least $|i - j|$ and the distance to $x_i$ is at least 2. Hence, by Lemma~\ref{lem:hamming}, the total probability over all parents we can choose (each with probability $\frac{1}{n + 1}$) that we generate $x_i'$ is at most 

    \begin{align*}
        \sum_{j \ne i} \frac{1}{n + 1} &\cdot \frac{1}{e(n - 1)^{|j - i|}} + \frac{1}{n + 1} \cdot \frac{1}{e(n - 1)^2} \\
        &\le \frac{2}{e(n + 1)} \left(\frac{1}{(n - 1)^2} +  \sum_{d = 1}^n \frac{1}{(n - 1)^d} \right) \\
        &\le \frac{2}{e(n + 1)} \left(\frac{1}{(n - 1)^2} +  \frac{1}{n - 1} \cdot \frac{1}{1 - \frac{1}{n - 1}} \right) \\
        &= \frac{2}{en^2}\left(1 + O\left(\frac{1}{n}\right)\right).
    \end{align*}

    By the union bound over all $i \in I_{00}(P_t)$, the probability that we replace any individual $x_i$ with a bit string different from $x_{i + 1}$ only in the hot position is at most
    \begin{align*}
        |I_{00}(P)_t| \frac{2}{en^2}\left(1 + O\left(\frac{1}{n}\right)\right) &\le \frac{n - 1}{2} \cdot \frac{2}{en^2}\left(1 + O\left(\frac{1}{n}\right)\right) \\
        &\le \frac{1}{en} \left(1 + O\left(\frac{1}{n}\right)\right),
    \end{align*} 
    which is an upper bound on $\Pr[A_{00}]$.

    Summing up the probabilities of $A_{00}$, $A_{11}$ and $A_{10}$, we have that the probability of $A$ is at most 
    \begin{align*}
        \Pr[A] &\le \frac{41}{2en}\left(1 + O\left(\frac{1}{n}\right)\right) + \left(\frac{13}{2n} + |J_{10}(P_t)| \cdot \frac{2}{n}\right)\left(1 + O\left(\frac{1}{n}\right)\right) \\
        & + \frac{1}{en} \left(1 + O\left(\frac{1}{n}\right)\right) 
        = \left(\frac{43 + 13 e}{2en} + |J_{10}(P)| \cdot \frac{2}{n}\right)\left(1 + O\left(\frac{1}{n}\right)\right)  \\
        &\le \frac{15 + 2|J_{10}(P)|}{en} \left(1 + O\left(\frac{1}{n}\right)\right),
    \end{align*}
    since $\frac{43 + 13e}{2e} \approx 14.409$.
\end{proof}

We are now in position to prove the main result of this subsection.

\begin{lemma}\label{lem:state2}
    For any iteration $t$ with  $P_t$ in State 2 the expected time until the end of the current phase is at most $320en + O(1)$ iterations. The probability that at the end of the phase we have a population with optimal diversity at least $\frac{1}{3551232en + O(1)} = \Omega(\frac{1}{n})$.    
\end{lemma}
\begin{proof}
    Similar to the proof of Lemma~\ref{lem:state3}, we aim at applying Lemma~\ref{lem:phase}. W.l.o.g. we assume that we start at iteration $t = 1$ with $P_1$ in State 2. We consider the sequence of algorithm's iterations as a sequence $\{X_t\}_{t \in \N}$ of trials. Each trial $X_t$ has three possible outcomes. The first outcome $\omega_1$ occurs either when we get $P_{t + 1}$ with optimal diversity or when we get $P_{t + 1}$ in State 3 and the phase ens in a population with optimal diversity. In the latter case we consider all iterations spent in State 3 as an auxiliary trial not belonging to the sequence $\{X_t\}_{t \in \N}$. The second outcome $\omega_2$ occurs when we either end the phase without getting $P_{t+ 1}$ with optimal diversity or when we get $P_{t + 1}$ in State 3 and then end a phase without finding the optimal population (in this case we also consider all iterations spent in State 3 as an auxiliary trial). The last outcome $\omega_3$ occurs when neither the phase ends, nor we get $P_{t + 1}$ in State $3$.

    For all $t \in \N$ we define events $A_t$ and $B_t$ as $X_t = \omega_1$ and $X_t = \omega_2$ correspondingly and we define $C_t$ as an event when for all $\tau \in [1..t - 1]$ we have $X_t \in \tau$ (this is undefined for $t = 1$). Then if $P_1$ is in State 2 and for some $t \ge 2$ event $C_t$ occurs, then it means that we have $P_t$ in State 2 as well, since we do not leave State 2 and we also do not end the phase in any iteration $\tau \in [1..t - 1]$.

    We define $p_t = \Pr[A_t \mid C_t]$ and $q_t = \Pr[B_t \mid C_t]$. To apply Lemma~\ref{lem:phase}, we need to estimate $p_t + q_t$ and $\frac{q_t}{p_t}$.

    We first estimate $p_t + q_t$. Since events $A_t$ and $B_t$ are disjoint, this is $\Pr[A_t \cup B_t \mid C_t]$. If we have $|I_{10}(P_t)| \le \frac{n}{32}$, then by Lemma~\ref{lem:state2-to-state3} the probability to go to State 3, which is a sub-event of $A_t \cup B_t$, is at least $\frac{1}{256en}$. Thus, we have $\Pr[A_t \cup B_t \mid C_t] \ge \frac{1}{256en}$. Otherwise, if we have $|I_{10}(P_t)| > \frac{n}{32}$, then by Lemma~\ref{lem:h1c0-prob} the probability to replace an individual with index in $|I_{10}(P_t)|$ is at least $\frac{1}{2en^2}(1 - O(\frac{1}{n}))$. Replacing one of 
    at least $\frac{n}{32}$ such individuals is a sub-event of $A_t \cup B_t$, and its probability is at least
    \begin{align*}
        \frac{n}{32} \cdot \frac{1}{2en^2} \left(1 - O\left(\frac{1}{n}\right)\right) = \frac{1}{64en}\left(1 - O\left(\frac{1}{n}\right)\right) > \frac{1}{256en},
    \end{align*}
    if $n$ is large enough. Hence, independently on $|I_{10}(P_t)|$ we always have $\Pr[A_t \cup B_t \mid C_t] \ge \frac{1}{256en} \eqqcolon p$.

    To find an upper bound on $\frac{q_t}{p_t}$ we fix an arbitrary $t$ and split the event $A_t \cup B_t$ into the following three disjoint sub-events (we avoid $t$ in the notation, since $t$ is fixed).
    Let $\Dopt$ be the event when $P_{t + 1}$ has the optimal diversity, $\Ds$ be the event when $P_{t + 1}$ is in State 3 and let $\Deop$ be the event that we end the phase, but $P_{t + 1}$ is not in State 3 and it is not optimally diverse. 
    We denote the probabilities of these events $\popt$, $\ps$ and $\peop$ respectively. If $\Dopt$ occurs, then iteration $t$ results in $X_t = \omega_1$. If $\Deop$ occurs, then $X_t = \omega_2$. If $\Ds$ occurs, then by Lemma~\ref{lem:state3} with probability at least $\frac{1 - O(\frac{1}{n})}{1156n}$ we have $X_t = \omega_1$ and otherwise we have $X_t = \omega_2$.

    We denote by $q \coloneqq \frac{1 - O(\frac{1}{n})}{1156n}$ the lower bound on the probability of $\omega_1$ conditional on $\Ds$. Then we have
    \begin{align}\label{eq:P-state2}
        \frac{q_t}{p_t} = \frac{q_t + p_t}{p_t} - 1 \le \frac{\popt + \ps + \peop}{\popt + \ps q} - 1 = \frac{\peop + \ps(1 - q)}{\popt + \ps q}.
    \end{align}
    We consider two cases depending on population $P_t$.

    \textbf{Case 1:} we have either $|I_{10}(P_t)| \ge \frac{n}{32}$ or $|J_{10}(P_t)| \ge 1$. If $\Deop$ occurs, then we either have $|\Jhot(P_{t + 1})| > |\Jhot(P_t)|$ or we have changed an individual with index in $|I_{10}(P_t)|$. Similar to the notation used in Lemma~\ref{lem:state3}, we denote $m \coloneqq |I_{10}(P_t)|$ and $m' \coloneqq |J_{10}(P_t)|$. Then by Lemma~\ref{lem:end-phase-state2} we have 
    \begin{align*}
        \frac\peop\popt &\le \frac{\frac{12}{n} + \frac{2}{n}m'}{\left(\frac{m - m'}{en^3} + \frac{m'}{en^2}\right)\left(1 - O\left(\frac{1}{n}\right)\right)} \\
        &= \frac{en^2}{n} \cdot \frac{(12 + 2m')}{(\frac{m}{n} + m')} \left(1 + O\left(\frac{1}{n}\right)\right) \\
        &= en \cdot \frac{(12 + 2m')}{(\frac{m}{n} + m')} \left(1 + O\left(\frac{1}{n}\right)\right),
    \end{align*}
    where we hid the $\frac{-m'}{en^3}$ term in the denominator in the first line into the $O(\frac{1}{n})$ term. If we have $m \ge \frac{n}{32}$, then this is at most
    \begin{align*}
        en \cdot \frac{12 + 2m'}{\frac{1}{32} + m'}\left(1 + O\left(\frac{1}{n}\right)\right) &= 2en\left(1 + \frac{6 - \frac{1}{32}}{\frac{1}{32} + m'}\right)\left(1 + O\left(\frac{1}{n}\right)\right) \\
        &\le 386en + O(1).
    \end{align*}
    If we have $m' \ge 1$, then this is at most 
    \begin{align*}
        en \cdot \left(\frac{12}{m'} + 2\right) \left(1 + O\left(\frac{1}{n}\right)\right) \le 14en + O(1).
    \end{align*}
    In both cases this is at most $386en + O(1) \approx 1049n$, which is smaller than $\frac{1 - q}{q} = \frac{1156n(1 - \Theta(\frac{1}{n}))}{1 - O(\frac{1}{n})} = 1156n \pm O(1)$, when $n$ is large enough. Hence, we have $\peop \le \frac{1 - q}{q} \popt$, and by~\eqref{eq:P-state2} we obtain
    \begin{align*}
        \frac{q_t}{p_t} \le \frac{\popt \cdot \frac{1 - q}{q} + \ps (1 - q)}{\popt + \ps q} = \frac{1 - q}{q} \le 1156n + O(1).
    \end{align*}

    \textbf{Case 2:} $|I_{10}(P)| < \frac{n}{32}$ and $|J_{10}(P)| = 0$. By Lemma~\ref{lem:state2-to-state3} we have $\ps \ge \frac{1}{256en}$. By Lemma~\ref{lem:end-phase-state2} we have $\peop \le \frac{12}{n} \le 12 \cdot 256e\ps$. Hence, by~\eqref{eq:P-state2} we have

    \begin{align*}
        \frac{q_t}{p_t} &\le \frac{3072e\ps + (1 - q)\ps}{q \ps} \le \frac{3072e + 1}{q} \\
        &= (3072e + 1)(1156n +O(1)) = 3551232en + O(1),
    \end{align*}
    which is a larger bound than in case 1. Hence, in both cases we have $\frac{q_t}{p_t} \le 3551232en + O(1) \eqqcolon \alpha$. 
    
    Consequently, by Lemma~\ref{lem:phase} the probability that the current phase ends in the optimal population is at least $\frac{1}{1 + \alpha} = \frac{1}{3551232en + O(1)}$ and that the expected time until event $A_t \cup B_t$ occurs is at most $\frac{1}{p} \le 256en$. When $A_t \cup B_t$ occurs, then we either end the phase immediately or we go to State 3, where by Lemma~\ref{lem:state3} we need in expectation another $64en + O(1)$ iterations to end the phase. Therefore, the expected time until the phase ends is at most $320en + O(1)$. \qedhere
\end{proof}

\subsection{Analysis of State 1}

In this section we show, that we spend in State $1$ expected number of $O(n)$ iteration before we leave it to get either to another state or to the optimal population. For this it is enough to show the $\Omega(\frac{1}{n})$ probability to leave State 1, which we do in the following lemma.

\begin{lemma}\label{lem:state1-to-state2}
    Let $n \ge 5$. If at iteration $t$ we have $P_t$ in State 1, then the probability that $|\Jhot(P_{t + 1})| \le 17$ is at least $\frac{1}{256en}$.
\end{lemma}
\begin{proof}
    This proof is very similar to the proof of Lemma~\ref{lem:state2-to-state3}, but it is easier, since we do not need to have $|I_{10}(P_{t + 1})| \ge \frac{n}{32}$. We call a balanced position in $P_t$ a \emph{hot-candidate} (analogue to the cold-candidates positions in Subsection~\ref{sec:state2}), if there are at least $\frac{n}{16}$ individuals with index in $I_{H1}(P)$ which have a zero-bit in that position.

    To move the hot position to a hot-candidate position $i$, we can choose one of the $\frac{n}{16}$ individuals that have a zero-bit in position $i$ and a one bit in the hot position and flip these two bits without flipping any other bit. It will give us an individual with the same number of one-bits as in its parent and if it replaces the parent, then we balance the hot position and add an extra one-bit to position $i$, which makes it hot. Thus, this does not reduce the diversity and is accepted by the \gsemod. The probability that we do that is at least
    \begin{align*}
        \frac{n}{16} \cdot \frac{1}{(n + 1)} \cdot \frac{1}{n^2} \left(1 - \frac{1}{n}\right)^{n - 2} \ge \frac{1}{16en^2}.
    \end{align*}

    We now show that there are many (at least $\frac{n}{8}$) hot-candidates positions. Consider all the individuals with a one-bit in the hot position, that are, the individuals with index in $I_{H1}(P_t)$. There are $\frac{n + 3}{2}$ such individuals, and since they all have a different fitness, they have at least
    \begin{align*}
        \sum_{i = 0}^{\frac{n + 1}{2}} i = \frac{(n + 1)(n + 3)}{8}
    \end{align*} 
    zero-bits in them. Since at most $\frac{n + 3}{2}$ of these bits are in the cold position, at least 
    \begin{align*}
        \frac{(n + 1)(n + 3)}{8} - \frac{n + 3}{2} = \frac{n^2 - 9}{8}
    \end{align*}
    zero-bits are in the balanced positions. Assume that there are $s < \frac{n}{8}$ hot-candidate positions. Since each of these positions is balanced, none of them can have more than $\frac{n + 1}{2}$ zero-bits in it. The other positions, since they are not hot-candidates, have at most $\frac{n}{16} - 1$ zero-bits in them. Hence by the same arguments as in the proof of Lemma~\ref{lem:cold-candidates-number}, the total number of zero-bits in the individuals with a one-bit in the hot position is
    \begin{align*}
        s \cdot \frac{n + 1}{2} + (n - 2 - s) \cdot \left(\frac{n}{16} - 1\right) < \frac{n^2 - 9}{8},
    \end{align*} 
    when $n \ge 5$. Hence, we must have $s \ge \frac{n}{8}$.

    We now show that at least $\frac{n}{16}$ of these hot-candidates are such that when we move the hot position to them we get $|\Jhot(P_{t + 1})| \le 17$. Consider an event when we move the hot position to position $j$ (which is not necessarily a hot-candidate) by replacing individual $x_i$ with $x_i'$. Consider also some index $k \in [0..n]$. There are two cases, in which $k$ belongs to $\Jhot(P_{t + 1})$, namely, (i) if $k = i$ and $x_i'$ is different from $x_{i - 1}$ only in position $j$ or (ii) if $k \ne i$ and $k \ne i + 1$, and $x_k$ was different from $x_{k - 1}$ only in position $j$ in $P_t$. Note that $k = i + 1$ can never be in $\Jhot(P_{t + 1})$, since $x_i'$ must have a one-bit in position $j$ (moving the hot position to $j$ implies that we increase the number of one-bits in it), and thus it cannot be different from $x_{i + 1}$ only in the new hot position $j$. 
    
    With this observation, for each position $j \in [1..n]$, except for the hot and the cold positions, we denote by $J_j$ the set of indices $k$ such that $x_k$ is different from $x_{k - 1}$ only in position $j$ in $P_t$. Note that these are disjoint sets for different positions $j$ and they are independent from the index of the individual we replace. Then for all $j \in [1..n]$ (except the two almost balanced positions) moving the hot position to $j$ will give us $|\Jhot(P_{t + 1})| \le |J_j| + 1$, where we add one to take into account case (i), when the index of the changed individual is added to $\Jhot(P_{t + 1})$.

    Since $J_j$ are disjoint subsets of $[1..n]$ (note that $k = 0$ cannot be in $J_j$ by the definition of $J_j$), then we have $\sum_j |J_j| \le n$. If we assume that there are at least $\frac{n}{16}$ positions $j$ that have $|J_j| > 16$, then we have $\sum_j |J_j| > \frac{n}{16} \cdot 16 = n$, hence we must have less than $\frac{n}{16}$ such positions. Consequently, there are at least $\frac{n}{8} - \frac{n}{16}$ hot-candidates $j$ with $|J_j| \le 16$, moving the hot position to which yields $|\Jhot(P_{t + 1})| \le 16 + 1 = 17$. The probability to move the hot position to any of these $\frac{n}{16}$ hot-candidates is at least 
    \begin{align*}
        \frac{n}{16} \cdot \frac{1}{16en^2} = \frac{1}{256en}.
    \end{align*}
\end{proof}

We show the main result of this subsection in the following corollary.

\begin{corollary}\label{cor:state1-end-of-phase}
    Consider some iteration $\tau$ with $P_\tau$ in State 1. Then the expected time until the end of phase is at most $576en + O(1)$ iterations and the probability that the phase ends in an optimal population is at least $\frac{1}{3551232en}(1 - O(\frac{1}{n}))$.
\end{corollary}

\begin{proof}[Proof of Corollary~\ref{cor:state1-end-of-phase}]
    W.l.o.g. assume that $\tau = 1$.
    For all $t \in \N$ let $A_t$ be the event that either $P_{t}$ has an optimal diversity or it is in State 2 or it is in State 3. Let $T$ be the minimum $t$ when $A_t$ occurs. For all $t < T$ we have $P_t$ in State 1 and the phase does not end earlier than in the end of iteration $T$. 
    
    For all $t$ event $A_t$ occurs when we get $|\Jhot(P_{t + 1})| \le 17$. By Lemma~\ref{lem:state1-to-state2} the probability of this is at least $\frac{1}{256en}$, therefore, $T$ is dominated by the geometric distribution $\Geom(\frac{1}{256en})$, and thus $E[T] \le 256en$.

    At iteration $T$ there are three options. First, we might have $P_{T + 1}$ with the optimal diversity and therefore, the phase ends immediately and the probability that it ends in an optimal population is $1$. Second, we can have $P_{T + 1}$ in State 2, then by Lemma~\ref{lem:state2} we need another $320en +O(1)$ iterations in expectation to end the phase and the phase ends by finding a population with optimal diversity with probability at least $\frac{1}{3551232en + O(1)}$. The last option is that we have $P_{T+ 1}$ in State 3 and then by Lemma~\ref{lem:state3} we need another $64en + O(1)$ iterations until we end the phase and in the end of the phase we find an optimal population with probability $\frac{1 - O(\frac{1}{n})}{1156n}$. In all three cases the expected number of iterations until we end the phase is at most $320en +O(1)$, thus, when we are in State 1, the expected number of iterations until the end of the phase is at most $E[T] + 320en +O(1) \le 576en + O(1)$. The probability that the phase ends in an optimal population is at least $\frac{1}{3551232en + O(1)}$.
\end{proof}

\subsection{The Total Runtime}
\label{sec:total-runtime}

In this subsection we prove our main result, that is, Theorem~\ref{thm:runtime}.

\begin{proof}[Proof of Theorem~\ref{thm:runtime}]
    Let $T_i$ be the time of the $i$-th phase of the algorithm and let $N$ be the number of the first successful phase which ends in an optimal population. Then the total runtime $T$ of the algorithm is $T = \sum_{i = 1}^N T_i$.
    We aim at showing that $T$ is integrable and to bound its expectation from above. Since $T$ is a sum of non-negative random variables, then it is also non-negative, hence we have $E[|T|] = E[T]$, and it is enough to only give an upper bound on $E[T]$, since it will also imply that $T$ is integrable. By the law of total expectation we have
    \begin{align}\label{eq:last1}
        \begin{split}
            E[T] &= \sum_{k = 1}^{\infty} \Pr[N = k] E\left[\sum_{i = 1}^N T_i ~\bigg\vert~ N = k\right] \\
            &= \sum_{k = 1}^{\infty} \Pr[N = k] E\left[\sum_{i = 1}^k T_i ~\bigg\vert~ N = k\right]
        \end{split}
    \end{align}

    By Lemmas~\ref{lem:state3} and~\ref{lem:state2} and by Corollary~\ref{cor:state1-end-of-phase}, we have that conditional on $N = k$  for all $i \le k$ we have $E[X_i \mid N = k] \le 576en + O(1) \eqqcolon \Delta$. Hence, for every $k \in \N$ we have
    \begin{align}\label{eq:last2}
        E\left[\sum_{i = 1}^k T_i ~\bigg\vert~ N = k\right] = \sum_{i = 1}^k E[T_i \mid N = k] \le k\Delta
    \end{align}

    Therefore, we have 
    \begin{align}\label{eq:last3}
        E[T] \le \sum_{k = 1}^{\infty} \Pr[N = k] k\Delta = \Delta E[N].
    \end{align}
    
    By Lemmas~\ref{lem:state3} and~\ref{lem:state2} and by Corollary~\ref{cor:state1-end-of-phase} we also have that $N$ is dominated by the geometric distribution $\Geom(\frac{1}{3551232en + O(1)})$, thus $E[N] \le 3551232en + O(1)$. Consequently, we have 
    \begin{align*}
        E[T] \le \delta E[N] \le (576en + O(1))(3551232en + O(1)) = O(n^2).
    \end{align*}
    This is a finite upper bound, which implies the correctness of Eqs.~\eqref{eq:last3}, \eqref{eq:last2} and then~\eqref{eq:last1} (in that order).\qedhere

\end{proof}

\section{Conclusion}
\label{sec:conclusion}

In this paper we have analysed the last stage of the optimization of the total Hamming distance on \oneminmax with \gsemod. We have shown that the population of \gsemod performs a random walk, and the rigorous study of this random walk reveals that a significant part of it is spent in a "good" region of the populations' space, where we have a lot of opportunities to make progress. This is in a big contrast with the previous study~\cite{DBLP:conf/gecco/DoerrGN16}, which pessimistically assumes that we always have the minimal number of such opportunities. We show that the pessimism is too critical in the last stage and increases the upper bound by a factor of $\Omega(n)$. This also indicates that the pessimism of~\cite{DBLP:conf/gecco/DoerrGN16} is too strong in the earlier stages, that is, during those stages the population is also likely to perform a random walk, which often visits good regions of the populations' space. Our study suggests the properties of the population which should be considered, shows rigorously how the population can change during the random walk and suggests some suitable (but not novel) methods for the analysis of this random walk, such as dividing it into phases. We are optimistic that these observations can be also helpful outside of the EDO context, e.g., in the analysis of population-based algorithms.

Intuition suggests that having more opportunities to improve the diversity in the earlier stages of the optimization should result into a coupon collector effect and give the expected runtime of $O(n^2\log(n))$ iterations to optimize the diversity form scratch.
Taking into account that the \gsemod finds a population covering the whole Pareto front of \oneminmax starting from a random point in $O(n^2\log(n))$ expected time (which was shown in~\cite{GielL10}, Theorem~3), this would suggest that optimizing the diversity is not asymptotically harder than covering the Pareto front (however, this statement also requires a lower bound on the time to cover the front).

In general we are optimistic that the analysis of random walks of complicated sets of solutions provided in this paper might be fruitful when studying the EDO on different problems. We also note that the leading constants in our results are unnaturally large due to the pessimistic assumptions we made in our proofs. However, our preliminary empirical study suggests that during a typical run we do not have many individuals in distance one from the individuals in the neighboring fitness levels, and therefore, in our analysis we could disregard the presence of individuals in $J_{10}$, $J_{00}$ and $J_{11}$. This would allow us to simplify our analysis and would significantly reduce the leading constants. For this reason a more rigorous empirical study of this problem is one of the most interesting future directions of this topic. 

\section*{Acknowledgements}
This work was supported by the Australian Research Council through grants DP190103894 and FT200100536.

\end{document}